[\def\year{2021}\relax
\documentclass[letterpaper]{article} 
\usepackage{aaai21}  
\usepackage{times}  
\usepackage{helvet} 
\usepackage{courier}  
\usepackage[hyphens]{url}  
\usepackage{graphicx} 
\usepackage{natbib}  
\usepackage{caption} 
\usepackage[switch]{lineno}

\usepackage[utf8]{inputenc} 
\usepackage[T1]{fontenc}    
\usepackage{hyperref}       
\usepackage{url}            
\usepackage{booktabs}       
\usepackage{dcolumn}
\usepackage{amsfonts}       
\usepackage{nicefrac}       
\usepackage{microtype}      
\usepackage{subcaption}
\usepackage{amsmath}
\usepackage{amssymb}
\usepackage{amsthm}
\usepackage{amsfonts}
\usepackage{mathrsfs} 
\usepackage{subcaption}
\usepackage{paralist}
\usepackage[shortlabels]{enumitem}
\usepackage{algorithm}
\usepackage[noend]{algpseudocode}

\usepackage{multirow} 
\usepackage{mathtools} 
\usepackage{etoolbox} 
\usepackage{siunitx} 
\usepackage{xcolor}
\usepackage{scalerel,stackengine}
\stackMath
\newcommand\reallywidehat[1]{%
\savestack{\tmpbox}{\stretchto{%
  \scaleto{%
    \scalerel*[\widthof{\ensuremath{#1}}]{\kern-.6pt\bigwedge\kern-.6pt}%
    {\rule[-\textheight/2]{1ex}{\textheight}}
  }{\textheight}%
}{0.5ex}}%
\stackon[1pt]{#1}{\tmpbox}%
}

\newcommand{\R}[0]{\mathbb{R}}

\newcommand{\diag}[0]{\operatorname{diag}}

\newcommand{\dps}{\displaystyle}

\newtheorem{theorem}{Theorem}

\newtheorem{definition}{Definition}
\newtheorem{lemma}{Lemma}
\newtheorem{remark}{Remark}
\newtheorem{example}{Example}

\newtheorem{problem}{Problem}
\newtheorem{corollary}{Corollary}

\DeclareMathOperator*{\argmax}{arg\,max}

\DeclareMathOperator{\supp}{supp}

\newcommand{\mypar}[1]{\noindent \textbf{#1.}}
\newcommand{\cited}[1]{\citeauthor{#1}\ \shortcite{#1}}

\newcommand{\mbf}[1]{\mathbf{#1}}
\newcommand{\mydef}{=}
\newcommand{\at}[1]{\left(#1\right)} 
\newcommand{\set}[1]{\{#1\}} 
\newcommand\srto[2]{#1\downarrow_{#2}}
\newcommand\hsrto[2]{\reallywidehat{#1\downarrow_{#2}}}
\newcommand{\prob}[1]{\mathbb{P}\left(#1\right)}

\renewcommand{\S}{\mathcal{S}}
\newcommand{\dcup}{\dot{\cup}}
\newcommand{\D}{\mathcal{D}}
\newcommand{\K}{\mathcal{K}}
\renewcommand{\P}{\mathcal{P}}
\newcommand{\B}{\mathcal{B}}
\newcommand{\A}{\mathcal{A}}

\makeatletter
\newenvironment{ssft}[1][htb]{%
    \renewcommand{\ALG@name}{SSFT}
   \begin{algorithm}[#1]%
  }{\end{algorithm}}
\makeatother

\makeatletter
\newenvironment{ssft2}[1][htb]{%
    \renewcommand{\ALG@name}{SSFT+}
   \begin{algorithm}[#1]%
  }{\end{algorithm}}
\makeatother

\title{Learning Set Functions that are Sparse in Non-Orthogonal Fourier Bases}

\author{
	Chris Wendler,
	Andisheh Amrollahi,
	Bastian Seifert,
	Andreas Krause, 
	Markus Püschel\\
}
\affiliations{
\small{wendlerc@ethz.ch, amrollaa@ethz.ch, baseifert@ethz.ch, krausea@ethz.ch, pueschel@inf.ethz.ch}\\
Department of Computer Science\\
ETH Zurich, Switzerland}

\begin{document}


\maketitle

\begin{abstract}
Many applications of machine learning on discrete domains, such as learning preference functions in recommender systems or auctions, can be reduced to estimating a set function that is sparse in the Fourier domain. In this work, we present a new family of algorithms for learning Fourier-sparse set functions. They require at most $nk - k \log_2 k + k$ queries (set function evaluations), under mild conditions on the Fourier coefficients, where $n$ is the size of the ground set and $k$ the number of non-zero Fourier coefficients. 
In contrast to other work that focused on the orthogonal Walsh-Hadamard transform (WHT), our novel algorithms operate with recently introduced {\em non-orthogonal Fourier transforms} that offer different notions of Fourier-sparsity. These naturally arise when modeling, e.g., sets of items forming substitutes and complements. We demonstrate effectiveness on several real-world applications.
\end{abstract}

\section{Introduction}

Numerous problems in machine learning on discrete domains involve learning {\em set functions}, i.e., functions $s:2^N\to\R$ that map subsets of some ground set $N$ to the real numbers.  In recommender systems, for example, such set functions express diversity among sets of articles and their relevance w.r.t.~a given need \cite{sharma2019learning, balog2019transparent}; in sensor placement tasks, they express the informativeness of sets of sensors \cite{krause2008near}; in combinatorial auctions, they express valuations for sets of items \cite{brero2019machine}. A key challenge is to estimate $s$ from a small number of observed evaluations. Without structural assumptions an exponentially large (in $n=|N|$) number of queries is needed. Thus, a key question is which families of set functions can be efficiently learnt, while capturing important applications. One key property is {\em sparsity in the Fourier domain} \cite{stobbe2012learning, amrollahi2019efficiently}. 

The Fourier transform for set functions is classically known as the orthogonal Walsh-Hadamard transform (WHT) \cite{bernasconi1996fourier, de2008brief,  li2015active, cheraghchi2017nearly}. Using the WHT, it is possible to learn functions with at most $k$ non-zero Fourier coefficients with $O(nk)$ evaluations \cite{amrollahi2019efficiently}. In this paper, we consider an alternative family of {\em non-orthogonal} Fourier transforms, recently introduced in the context of discrete signal processing on set functions (DSSP) \cite{puschel2018discrete, puschel2020discrete}. In particular, we present the {\em first efficient algorithms} which (under mild assumptions on the Fourier coefficients), efficiently learn $k$-Fourier-sparse set functions requiring at most $(n+1)k - k \log_2 k$ evaluations. In contrast, naively computing the Fourier transform requires $2^n$ evaluations and $n 2^{n-1}$ operations \cite{puschel2020discrete}. 

Importantly, sparsity in the WHT domain does {\em not} imply sparsity in the alternative Fourier domains we consider, or vice versa. Thus, we significantly expand the class of set functions that can be efficiently learnt. One natural example of set functions, which are sparse in one of the non-orthogonal transforms, but not for the WHT, are certain preference functions considered by \cited{djolonga2016variational} in the context of recommender systems and auctions. In recommender systems, each item may cover the set of needs that it satisfies for a customer. If needs are covered by several items at once, or items depend on each other to provide value there are substitutability or complementarity effects between the respective items, which are precisely captured by the new Fourier transforms \cite{puschel2020discrete}. Hence, a natural way to learn such set functions is to compute their respective sparse Fourier transforms.

\mypar{Contributions} In this paper we develop, analyze, and evaluate novel algorithms for computing the sparse Fourier transform under various notions of Fourier basis:
\begin{enumerate}
	\item We are the first to introduce an efficient algorithm to compute the sparse Fourier transform for the recent notions of non-orthogonal Fourier basis for set functions \cite{puschel2020discrete}. In contrast to the naive fast Fourier transform algorithm that requires $2^n$ queries and $n 2^{n-1}$ operations, our sparse Fourier transform requires at most $nk - k \log_2 k + k = O(nk - k\log k)$ queries and $O(n k^2)$ operations to compute the $k$ non-zero coefficients of a Fourier-sparse set function. The algorithm works in all cases up to a null set of pathological set functions.
	 \item We then further extend our algorithm to handle an even larger class of Fourier-sparse set functions with $O(n^2 k - n k \log k)$ queries and $O(n^2 k + k^2 n)$ operations using filtering techniques.
	 \item We demonstrate the effectiveness of our algorithms in three  real-world set function learning tasks: learning surrogate objective functions for sensor placement tasks, learning facility locations functions for water networks, and preference elicitation in combinatorial auctions. The sensor placements obtained by our learnt surrogates are indistinguishable from the ones obtained using the compressive sensing based WHT by \cited{stobbe2012learning}. However, our algorithm does not require prior knowledge of the Fourier support and runs significantly faster. The facility locations function learning task shows that certain set functions are sparse in the non-orthogonal basis while being dense in the WHT basis. In the preference elicitation task also only half as many Fourier coefficients are required in the non-orthogonal basis as in the WHT basis, which indicates that bidders' valuation functions are well represented in the non-orthogonal basis. 
\end{enumerate}

\emph{Because of the page limit, all proofs are in the supplementary material.}

\section{Fourier Transforms for Set Functions}\label{dssp}

We introduce background and definitions for set functions and
associated Fourier bases, following the discrete-set signal processing (DSSP) introduced by~\cite{puschel2018discrete, puschel2020discrete}. DSSP generalizes
key concepts from classical signal processing, including shift, convolution, and Fourier transform to the powerset domain. The approach follows a general procedure that derives these concepts from a suitable definition of the shift operation \cite{puschel2008algebraic}.

\mypar{Set functions} We consider a ground set $N = \{x_1, \dots, x_n\}$. An associated set function maps each subset of $N$ to a real value:
\begin{equation}
s: 2^N \to \mathbb{R}; A \mapsto s\at{A}.
\end{equation} Each set function can be identified with a
$2^n$-dimensional vector $\mbf{s} = (s\at{A})_{A \subseteq N}$ by
fixing an order on the subsets. We choose the lexicographic order on the set indicator vectors.

\mypar{Shifts} Classical convolution (e.g., on images) is associated with the translation operator. Analogously, DSSP considers different versions of "set translations", each yielding a different DSSP model, numbered 1--5. One choice is
\begin{equation}\label{eq:shift4}
    \text{(model 4)}\quad T_Q s\at{A} = s\at{A \cup Q},\text{ for }Q\subseteq N.
\end{equation}
The shift operators $T_Q$ are parameterized by the powerset monoid $(2^N, \cup)$, since
the equality $T_Q(T_R s) = T_{Q \cup R} s$ holds for all
$Q , R \subseteq N$, and $s \in \mathbb{R}^{2^N}$. 

\mypar{Convolutional filters} The corresponding linear, shift-equivariant convolution in model 4 is given by 
\begin{equation}\label{eq:setconv}
(h * s)\at{A} = \sum_{Q \subseteq N} h\at{Q} s\at{A \cup Q}.
\end{equation}
Namely, $(h * T_R s)\at{A} =
T_R (h * s)\at{A}$, for all $R \subseteq N$. Convolving with $h$ is a linear mapping called a {\em filter} and $h$ is also a set function. In matrix notation we have $\mbf{h * s} = H \mbf{s}$, where $H$ is the filter matrix.

\mypar{Fourier transform and convolution theorem} The Fourier transform (FT) simultaneously dia\-gonalizes all filters, i.e., the matrix $F H F^{-1}$ is diagonal for all filter matrices $H$, where $F$ denotes the matrix form of the Fourier transform. Thus, different definitions of set shifts yield different notions of Fourier transform. For the shift in \eqref{eq:shift4} the Fourier transform of $s$ takes the form
\begin{equation}\label{eq:s_hat}
\widehat{s}\at{B} = \sum_{A \subseteq N: A \cup B = N} (-1)^{|A \cap B|} s\at{A}
\end{equation}
with the inverse
\begin{equation}\label{eq:s}
    s\at{A} = \sum_{B \subseteq N: A \cap B = \emptyset} \widehat{s}\at{B}.
\end{equation}
As a consequence we obtain the convolution theorem
\begin{equation}\label{convth}
\widehat{(h*s)}\at{B} = \overline{h}\at{B} \widehat{s}\at{B}.
\end{equation}
Interestingly, $\overline{h}$ (the so-called frequency response) is computed differently than $\widehat{s}$, namely as
\begin{equation}
\overline{h}\at{B} = \sum_{A \subseteq N: A \cap B = \emptyset} h\at{A}.
\end{equation}

In matrix form, with respect to the chosen order of $\mbf{s}$, the Fourier transform and its inverse are
\begin{equation}\label{ftmat}
F = \begin{pmatrix}
0 & \phantom{-}1 \\
1 & -1 \\
\end{pmatrix}^{\otimes n} \quad \text{and} \quad
F^{-1} = \begin{pmatrix}
1 & 1 \\
1 & 0 \\
\end{pmatrix}^{\otimes n},
\end{equation}
respectively, in which $M^{\otimes n} = M \otimes \cdots \otimes M$ denotes the $n$-fold Kronecker product of the matrix $M$. Thus, the Fourier transform $\widehat{\mbf{s}} = F \mbf{s}$ and its inverse $\mbf{s} = F^{-1}\widehat{\mbf{s}}$ can be computed in $n 2^{n-1}$ operations. 

The columns of $F^{-1}$ form the Fourier basis and can be viewed as indexed by $B\subseteq N$. The $B$-th column is given by $\mbf{f}^B_A = \iota_{A \cap B = \emptyset}$, where $\iota_{A \cap B = \emptyset} = 1$ if $A \cap B = \emptyset$ and $\iota_{A \cap B = \emptyset} = 0$ otherwise. The basis is not orthogonal as can be seen from the triangular structure in \eqref{ftmat}.

\mypar{Example and interpretation} We consider a special class of preference functions that, e.g., model customers in a recommender system~\cite{djolonga2016variational}. Preference functions naturally occur in machine learning tasks on discrete domains such as recommender systems and auctions, in which, for example, they are used to model complementary- and substitution effects between goods. Goods complement each other when their combined utility is greater than the sum of their individual utilities. Analogously, goods substitute each other when their combined utility is smaller than the sum of their individual utilities. Formally, a preference function is given by
\begin{equation}\label{eq:fldc}
\begin{aligned}
p(A) &= \sum_{i \in A} u_i + \sum_{\ell=1}^L \left(\max_{i \in A} r_{\ell i} - \sum_{i \in A} r_{\ell i}\right) \\&\ \ - \sum_{k=1}^K \left(\max_{i \in A} a_{ki} - \sum_{i \in A} a_{ki}\right).
\end{aligned}
\end{equation}
Equation~\eqref{eq:fldc} is composed of a so-called modular term parametrized by $u \in \R^n$, a repulsive term parametrized by $r \in \R_{\geq 0}^{L \times n}$, with $L \in \mathbb{N}$, and an attractive term parametrized by $a \in \R_{\geq 0}^{K \times n}$, with $K \in \mathbb{N}$. The repulsive term captures substitution and the attractive term complementary effects. 

\begin{example}[Running example]\label{example} Consider the ground set $\{x_1, x_2, x_3\}$, in which $x_1$ represents a tablet, $x_2$ a laptop and $x_3$ a tablet pen. Now, we create a preference function with a substitution effect between the laptop and the tablet which is expressed in the repulsive term $r = (2, 2, 0)$ and a complementary effect between the tablet and the tablet pen which is expressed in the attractive term $a = (1, 0, 1)$. The individual values of the items are expressed in the modular term $u = (2, 2, 1)$. As a result we get the following preference function $p$ and also show its Fourier transform $\widehat{p}$:

\begin{center}
    \resizebox{0.98\columnwidth}{!}{
	\begin{tabular}{@{}
			c
			c
			c
			S[table-format=1.0]
			S[table-format=1.0]
			S[table-format=1.0]
			S[table-format=1.0]
			S[table-format=1.0]
			S[table-format=1.0]
			@{}}\toprule
		 & $\emptyset$ & $\set{x_1}$ & $\set{x_2}$ & $\set{x_1, x_2}$& $\set{x_3}$ & $\set{x_1, x_3}$ & $\set{x_2, x_3}$ & $\set{x_1, x_2, x_3}$\\
        \midrule
		$p$                & {0}& {\,\,2} & {2}  & {\,\,2}  & {\,\,1}  & {4} & {3} & {4}\\
		$\widehat{p}$      & {4}& {-1}    & {0}  & {-2}     & {-2}     & {1} & {0} & {0} \\ \bottomrule
	\end{tabular}
}
\end{center}

Namely, as desired, $p\at{\set{x_1, x_2}} < p\at{\set{x_1}}+ p\at{\set{x_2}}$ and $p\at{\set{x_1, x_3}} > p\at{\set{x_1}}+ p\at{\set{x_3}}$. Note that $\widehat{p}$ is sparse. Next, we show this is always the case.
\end{example}

\begin{lemma}\label{lem:fldc} Preference functions of the form \eqref{eq:fldc} are \emph{Fourier-sparse w.r.t.~model~4} with at most $1+ n + L n + K n$ non-zero Fourier coefficients.
\end{lemma}

Motivated by Lemma~\ref{lem:fldc}, we call set functions that are Fourier-sparse w.r.t.~model~4 \emph{generalized preference functions}. Formally, a generalized preference function is defined in terms of a collection of distinct subsets $N = \set{S_1, \dots, S_n}$ of some universe $U: S_i \subseteq U$, $i \in \set{1, \dots, n}$, and a weight function $w: U \to \R$. The weight of a set $S \subseteq U$ is $\mbf{w}(S) = \sum_{u \in S} w(u)$. Then, the corresponding \emph{generalized preference function} is
\begin{equation}\label{wf}
s: 2^N \to \R; A \mapsto \mbf{w}\left(\bigcup_{S_i \in A} S_i\right).
\end{equation} 

For non-negative weights $s$ is called a \emph{weighted coverage function}~\cite{krause2014submodular}, but here we allow general (signed) weights. Thus, generalized preference functions are \emph{generalized coverage functions} as introduced in \cite{puschel2020discrete}. Generalized coverage functions can be visualized by a bipartite graph, see Fig.~\ref{fig:coverage_bipartite}. In recommender systems, $S_i$ could model the customer-needs covered by item $i$. Then, the score that a customer associates to a set of items corresponds to the needs covered by the items in that set. Substitution as well as complementary effects occur if the needs covered by items overlap (e.g., $S_i \cap S_j \neq \emptyset$). 

Concretely, we observe in Fig.~\ref{fig:coverage_bipartite} that in our running example the tablet and the laptop share one need with a positive sign which yields the substitution effect, and the tablet and the tablet pen share a need with a negative sign which yields the complementary effect. 

\begin{figure}
    \centering
    \begin{subfigure}{0.23\textwidth}
        \begin{center}
            \vphantom{\includegraphics[width=0.9\textwidth]{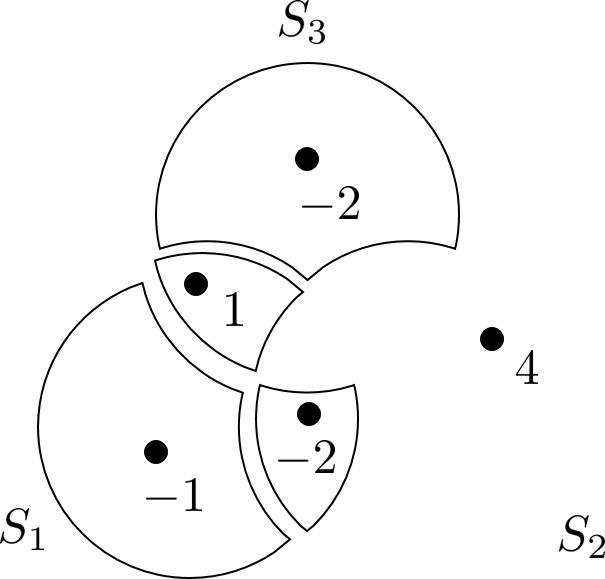}} 
            \includegraphics[width=0.6\textwidth]{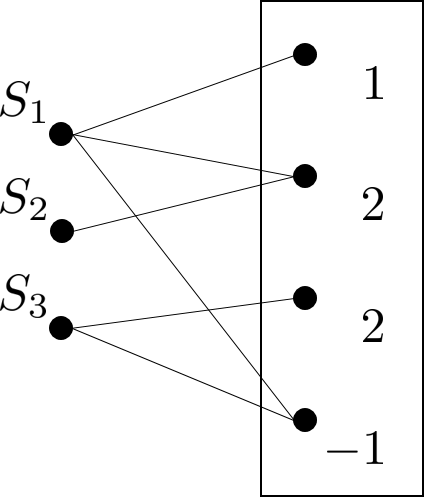}
        \end{center}
        \subcaption{Coverage function 
        }\label{fig:coverage_bipartite}
    \end{subfigure}
    \begin{subfigure}{0.23\textwidth}
        \begin{center}            
            \vphantom{\includegraphics[width=0.6\textwidth]{figures/coverage/bipartite.png}}
            \includegraphics[width=0.9\textwidth]{figures/coverage/venn.png}
        \end{center}
       \subcaption{Fourier transform in \eqref{eq:s_hat}
       }\label{fig:coverage_ft}
    \end{subfigure}
    \caption{Preference function $p$ from Example~\ref{example} as generalized coverage function  visualized as bipartite graph (a) and as Venn diagram (b). The Fourier coefficients are the weights of the fragments. Here, three are zero.}\label{fig:coverage}
\end{figure}

Interestingly, the Fourier coefficients for model 4 in \eqref{eq:s_hat} of a generalized coverage function are 
\begin{equation}\label{eq:spectrum}
\widehat{s}\at{B} = \begin{cases}
\sum_{u \in U} w(u), & \text{if } B = \emptyset,\\
-\mbf{w}\left(\bigcap_{S_i \in B} S_i \setminus \bigcup_{S_i \in N\setminus B} S_i\right), & \text{otherwise,}
\end{cases}
\end{equation}
which corresponds to the weights of the fragments of the Venn-diagram of the sets $S_i$ (Fig.~\ref{fig:coverage_ft}). If the universe $U$ contains fewer than $2^n$ items, some fragments will have weight zero, i.e., are Fourier-sparse. We refer to Section~VIII of \cited{puschel2020discrete} for an in-depth discussion and interpretation of the spectrum~\eqref{eq:spectrum}.

\begin{table}
\centering
	
	\resizebox{0.48\textwidth}{!}{
	\tiny
	$
	\begin{array}{@{}l llll@{}}\toprule
	 & \text{shift } T_Q s\at{A} & F \text{ (sum)}: \widehat{s}\at{B} = & F^{-1} \text{ (sum)}: {s}_A = \\ \midrule

	 3 & s\at{A \setminus Q} &  \dps\sum_{\stackrel{A\subseteq B}{\hphantom{ A \cup B = N}}} (-1)^{|A|}s\at{A} & 
	\hphantom{\tfrac{1}{2^{n}}} \dps\sum_{\stackrel{B\subseteq A}{\hphantom{A \cap B = \emptyset}}} (-1)^{|B|}\widehat{s}\at{B} 
	\\
	4 & s\at{A \cup Q} & \dps\sum_{\stackrel{A\subseteq N:}{ A \cup B = N}} (-1)^{|A \cap B|}s\at{A} & 
	\hphantom{\tfrac{1}{2^{n}}} \dps\sum_{\stackrel{B\subseteq N:}{A \cap B = \emptyset}} \widehat{s}\at{B}  \\
	5 & s\at{A \setminus Q \cup Q \setminus A} & \dps\sum_{\stackrel{A\subseteq N}{\hphantom{ A \cup B = N}}} (-1)^{|A \cap B|} s\at{A} & 
	\tfrac{1}{2^{n}}\dps\sum_{\stackrel{B\subseteq N}{\hphantom{A \cap B = \emptyset}}} (-1)^{|A \cap B|}\widehat{s}\at{B}  \\\bottomrule

	\end{array} 
	$}
	\caption{Shifts and Fourier concepts.\label{tab:sfdsp}}
\end{table}

\mypar{Other shifts and Fourier bases} As mentioned before, \cited{puschel2020discrete} considers 5 types of shift, i.e. DSSP models, each with its respective shift-equivariant
convolution, associated Fourier basis, and thus notion of Fourier-sparsity. Model 5 is the classical definition that yields the WHT and model 4 the version introduced above. Table~\ref{tab:sfdsp} collects the key concepts, also including model 3.

The notions of Fourier-sparsity differ substantially between models. For example, consider the coverage function for which there is only one element in the universe $U$ and this element is covered by all sets $S_1, \dots, S_n$. Then, $\widehat{s}\at{\emptyset} = 1$, $\widehat{s}\at{N} = -1$ and $\widehat{s}\at{B} = 0$ for $\emptyset \subset B \subset N$ w.r.t.~model~4, and $\widehat{s}\at{\emptyset} = 2^{n}-1$ and $\widehat{s}\at{B} = -1$ for all $\emptyset \subset B \subseteq N$ w.r.t.~the WHT. 

More generally, one can show that preference functions in \eqref{eq:fldc} with at least one row of pairwise distinct values in either the repulsive or attractive part are dense w.r.t.~the WHT basis.

The Fourier bases have appeared in different contexts before. For example, \eqref{eq:s_hat} can be related to the W-transform, which has been used by \cited{chakrabarty2012testing} to test coverage functions.

\section{Learning Fourier-Sparse Set Functions}

We now present our algorithm for learning Fourier-sparse set functions w.r.t.~model~4. One of our main contributions is that the derivation and algorithm are general, i.e., they also applies to the other models. We derive the variants for {\em models~3 and 5} from Table~\ref{tab:sfdsp} in the supplementary material. 

\begin{definition} A set function $s$ is called $k$-Fourier-sparse if \begin{equation}
   \supp(\widehat{s}) = \set{B: \widehat{s}\at{B} \neq 0}= \set{B_1, \dots, B_k}, 
\end{equation}
where we assume that $k$ is significantly smaller than $2^n$.
\end{definition}

Thus, exactly learning a $k$-Fourier-sparse set function is equivalent to
computing its $k$ non-zero Fourier coefficients and associated support. Formally, we want to solve:

\begin{problem}[Sparse FT]\label{prob:general} Given oracle access to query a $k$-Fourier-sparse set function $s$, compute its Fourier support and associated Fourier coefficients.
\end{problem}

\subsection{Sparse FT with Known Support}

First, we consider the simpler problem of computing the Fourier coefficients if the Fourier support $\supp(\widehat{s})$ (or a small enough superset $\mathcal{B} \supseteq \supp(\widehat{s})$) is known. In this case, the solution boils down to selecting queries $\mathcal{A} \subseteq 2^N$ such that the linear system of equations
\begin{equation}\label{eq:sampling}
\mbf{s}_{\mathcal{A}} = F^{-1}_{\mathcal{A}\mathcal{B}} \mbf{\widehat{s}}_{\mathcal{B}},
\end{equation} 
admits a solution. Here, $\mbf{s}_{\mathcal{A}} = (s\at{A})_{A \in \mathcal{A}}$ is the vector of queries, $F^{-1}_{\mathcal{A}\mathcal{B}}$ is the submatrix of $F^{-1}$ obtained by selecting the rows indexed by $\mathcal{A}$ and the columns indexed by $\mathcal{B}$, and $\mbf{\widehat{s}}_{\mathcal{B}}$ are the unknown Fourier coefficients we want to compute.

\begin{theorem}[\cited{puschel2020discrete}] \label{thm:sampling4} Let $s$ be $k$-Fourier-sparse with ${\supp(\widehat{s}) = \set{B_1, \dots, B_k} = \B}$. Let ${\A = \set{N \setminus B_1, \dots, N \setminus B_k}}$. Then $F^{-1}_{\mathcal{A}\mathcal{B}}$ is invertible and $s$ can be perfectly reconstructed from the queries $\mbf{s}_{\A}$.
\end{theorem}

Consequently, we can solve Problem~\ref{prob:general} if we have a way to discover a $\B\supseteq\supp(\widehat{s})$, which is what we do next.

\subsection{Sparse FT with Unknown Support}

In the following we present our algorithm to solve Problem~\ref{prob:general}. As mentioned, the key challenge is to determine the Fourier support w.r.t.~\eqref{eq:s_hat}. The initial skeleton is similar to the algorithm {\em Recover Coverage} by \cited{chakrabarty2012testing}, who used it to test coverage functions. Here we take the novel view of Fourier analysis to expand it to a sparse Fourier transform algorithm for all set functions. Doing so creates challenges since here the weight function in \eqref{wf} is not guaranteed to be positive. Using the framework in Section~\ref{dssp} we will analyze and address them.

Let $M\subseteq N$, and consider the associated restriction of a set function $s$ on $N$:
\begin{equation}
\srto{s}{2^M}: 2^M \to \R; A \mapsto s\at{A}
\end{equation}
The Fourier coefficients of $s$ and the restriction can be related as (proof in supplementary material):
\begin{equation}\label{eq:dsft4_restricted}
\hsrto{s}{2^M}\at{B} = \sum_{A \subseteq N \setminus M} \widehat{s}\at{A \cup B}.
\end{equation}
We observe that, if the Fourier coefficients on the right hand side of \eqref{eq:dsft4_restricted} do not cancel, knowing $\hsrto{s}{2^M}$ contains information about the sparsity of $\hsrto{s}{2^{M \cup \set{x}}}$, for $x \in N\setminus M$. To be precise, if there are no cancellations, the relation 
\begin{equation}
\begin{aligned}\label{eq:dsft4_sparsity_propagation}
\hsrto{s}{2^M}\at{B} 
&= \hsrto{s}{2^{M \cup \set{x}}}\at{B} + \hsrto{s}{2^{M \cup \set{x}}}\at{B \cup \set{x}}
\end{aligned}
\end{equation} 
implies that both $\hsrto{s}{2^{M \cup \set{x}}}\at{B}$ and $\hsrto{s}{2^{M \cup \set{x}}}\at{B \cup \set{x}}$ must be zero whenever $\hsrto{s}{2^M}\at{B}$ is zero. As a consequence, we can construct 
\begin{equation}
\B = \bigcup_{B \in \supp(\hsrto{s}{2^M})} \set{B, B \cup \set{x}},
\end{equation}
with $\supp(\hsrto{s}{2^{M \cup \set{x}}}) \subseteq \B$, from \eqref{eq:dsft4_sparsity_propagation}, and then compute $\hsrto{s}{2^{M \cup \set{x}}}$ with Theorem~\ref{thm:sampling4} in this case.

As a result we can solve Problem~\ref{prob:general} with our algorithm \textbf{SSFT}, under mild conditions on the coefficients that guarantee that cancellations do not occur, by successively computing the non-zero Fourier coefficients of restricted set functions along the chain
\begin{equation}\label{eq:dsft4_chain}
\srto{s}{2^{\emptyset}} = \hsrto{s}{2^{\emptyset}}, \hsrto{s}{2^{\set{x_1}}}, \hsrto{s}{2^{\set{x_1, x_2}}}, \dots, \hsrto{s}{2^{N}} = \widehat{s}.
\end{equation}

For example, \textbf{SSFT} works with probability one if all non-zero Fourier coefficients are sampled from independent continuous probability distributions:

\begin{lemma}\label{lem:cont_random_sf} With probability one \textbf{SSFT} correctly computes the Fourier transform of Fourier-sparse set functions $s$ with $\supp(\widehat{s}) = \mathcal{B}$ and randomly sampled Fourier coefficients, that satisfy
\begin{enumerate}
\item $\widehat{s}\at{B} \sim P_B$, where $P_B$ is a continuous probability distribution with density function $p_B$, for $B \in \supp(\widehat{s})$,
\item $p_{\mathcal{B}} = \prod_{B \in \mathcal{B}} p_B$.
\end{enumerate} 
\end{lemma}

\begin{figure*}[t!]
    \begin{minipage}[t]{0.5\textwidth}
        \null
\begin{ssft}[H]
\renewcommand{\thealgorithm}{}
\caption{Sparse set function Fourier transform of $s$}\label{alg:sdsft4}
\begin{algorithmic}[1]
	\State{$M_0 \gets \emptyset$}
	\State{$\hsrto{s}{2^{M_0}}\at{\emptyset} \gets s\at{\emptyset}$}
	\For{$i = 1, \dots, n$}
	\State{$M_i \gets M_{i-1} \cup \set{x_i}$}
	\State{$\mathcal{B} \gets \emptyset, \mathcal{A} \gets \emptyset$}
	\For{$B \in \supp(\hsrto{s}{2^{M_{i-1}}})$}
	\State{$\mathcal{B} \gets \mathcal{B} \cup \set{B, B \cup \set{x_i}}$}
	\State{$\mathcal{A} \gets \mathcal{A} \cup \set{M_i \setminus
            B, M_i \setminus (B \cup \set{x_i})}$}
	\EndFor
	\State{$\mbf{s}_{\mathcal{A}} \gets (s\at{A})_{A \in \mathcal{A}}$}
	\State{$\mbf{x} \gets \text{ solve } \mbf{s}_{\mathcal{A}} = F^{-1}_{\mathcal{A}\mathcal{B}} \mbf{x}$ for $\mbf{x}$}
	\For{$B \in \mathcal{B}$ with $\mbf{x}_B \neq 0$}
	\State $\hsrto{s}{2^{M_i}}\at{B} \gets \mbf{x}_B$
	\EndFor
	\EndFor
	\State \Return{$\hsrto{s}{2^{M_n}}$}
    \end{algorithmic}
\end{ssft}
\end{minipage}
\hfill
\begin{minipage}[t]{0.47\textwidth}
    \null
\begin{ssft2}[H]
\renewcommand{\thealgorithm}{}
\caption{Filtering based \textbf{SSFT} of $s$}\label{alg:sdsft4_general}
\begin{algorithmic}[1]
	\State{\em // Sample random coefficients.}
	\State{$h\at{\emptyset} = 1$
          \vphantom{{$\hsrto{s}{2^{M_0}}\at{\emptyset} \gets
              s\at{\emptyset}$}}} 
	\For{$x \in \{x_1, \dots, x_n\}$}
	\State{$h\at{\set{x}} \gets c \sim \mathcal{N}(0, 1)$}
	\EndFor
	\State{\em // Fourier transform of filtered set function.}
	\State{$\widehat{h * s} \gets $\textbf{SSFT}($h *
          s$)}
	\State{\em // Compute the original coefficients.}
	\For{$B \in \supp(\widehat{h * s})$}
	\State{$\widehat{s}_B \gets {\widehat{(h * s)}\at{B}}/{\overline{h}\at{B}}$}
	\EndFor	
	\State \Return{$\widehat{s}$}
        \Statex
        \Statex
        \Statex
    \end{algorithmic}
    \vspace{-0.2pt}
\end{ssft2}
\end{minipage}
\end{figure*}

\begin{figure}
    \centering
    
        \includegraphics[width=0.45\textwidth]{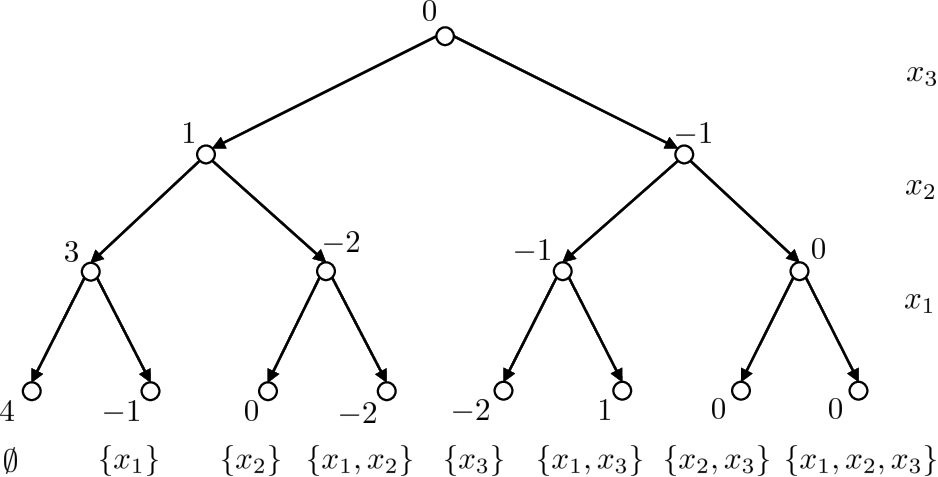}

    \caption{The Fourier coefficients of the restricted versions of $p$ from Example~\ref{example}. Depth 0 (root) corresponds to the Fourier transform of $\srto{p}{2^{\emptyset}}$, depth 1 to the one of $\srto{p}{2^{\{x_3\}}}$, depth 2 to the one of $\srto{p}{2^{\{x_2, x_3\}}}$ and depth 3 to the one of $p$.}\label{fig:tree}
\end{figure}

\mypar{Algorithm} We initialize \textbf{SSFT} with $M_0 = \emptyset$ and ${\srto{s}{2^{\emptyset}}\at{\emptyset}} = s\at{\emptyset}$ (lines~1--2). Then, in each iteration of the for loop (line~3), we grow our set $M_i = M_{i-1} \cup \set{x_i}$ by adding the next element (line~4), determine the superset $\B$ of $\supp(\hsrto{s}{2^{M_{i}}})$ based on the Fourier coefficients from the previous iteration $\hsrto{s}{2^{M_{i-1}}}$ (lines~5--8) and solve the resulting known-support-problem using Theorem~\ref{thm:sampling4} (lines~9--12). After $n$ iterations we end up with the Fourier coefficients of $s$.  

We depict an execution of \textbf{SSFT} on $p$ from Example~\ref{example} in Fig.~\ref{fig:tree}. We process the elements $x_1, x_2, x_3$ in reverse order, which results in the final Fourier coefficients being lexicographically ordered. Note that in this example applying \textbf{SSFT} naively would fail, as $\hsrto{p}{2^{\emptyset}}\at{\emptyset} = p\at{\emptyset} = 0$ \textbf{SSFT} would falsely conclude that both $\hsrto{p}{2^{\set{x_3}}}\at{\emptyset}$ and $\hsrto{p}{2^{\set{x_3}}}\at{\set{x_3}}$ are zero resulting in them and their children being pruned. Instead, we initialize the algorithm with $\hsrto{p}{\{x_3\}}$. The levels of the tree correspond to the elements of the chain~\eqref{eq:dsft4_chain}. When computing the Fourier coefficients of a level, the ones of the previous level determine the support superset $\B$. 
Whenever, \textbf{SSFT} encounters a Fourier coefficient that is equal to zero, all of its children are pruned from the tree. For example, $\hsrto{p}{2^{\set{x_2, x_3}}}\at{\set{x_2, x_3}} = 0$, thus, $\set{x_2, x_3}$ and $\set{x_1, x_2, x_3}$ are not included in the support superset $\B$ of $\hsrto{p}{N} = \widehat{p}$. While pruning in our small example only avoids the computation of $\hsrto{p}{2^{N}}\at{\set{x_2, x_3}}$ and $\hsrto{p}{2^N}\at{\set{x_1, x_2, x_3}}$, it avoids an exponential amount of computation in larger examples (for $|N| = n$, when $\hsrto{s}{2^{M_{i-1}}}\at{B} = 0$ the two subtrees of height $n - i$ containing its children are pruned). 

Note that for practical reasons we only process up to $k_\text{max}$ subsets in line 6. In line 11, we consider a Fourier coefficient $|\mbf{x}_B| < \epsilon$ (a hyperparameter) as zero.

\mypar{Analysis} We consider set functions $s$ that are $k$-Fourier-sparse (but not $(k-1)$-Fourier-sparse) with  support $\supp(\widehat{s}) = \set{B_1, \dots, B_k} = \mathcal{B}$, i.e., 
$\set{{s: 2^N \to \mathbb{R}}: {\widehat{s}\at{B} \neq 0} \text{ iff } B
  \in \mathcal{B}}$, which is isomorphic to
\begin{equation}
    \mathcal{S} =
\set{\widehat{\mbf{s}} \in \mathbb{R}^k: \widehat{\mbf{s}}_i \neq 0 \text{ for all } i \in \set{1, \dots, k}} .
\end{equation}

Let $\lambda$ denote the Lebesgue measure on $\mathbb{R}^k$. Let $\P^{M_i}_C = \set{B \in \B: B \cap M_i = C}$.

\mypar{Pathological set functions} \textbf{SSFT} fails to compute the Fourier coefficients for which $\hsrto{s}{2^{M_i}}\at{C} = 0$ despite $\P^{M_i}_C \neq \emptyset$. Thus, the set of pathological set functions $\D_1$, i.e., the set of set functions for which \textbf{SSFT} fails, can be written as the finite union of kernels 
\begin{equation}
\K_1(M_i, C) = \set{\widehat{\mbf{s}} \in \R^k: \hsrto{s}{2^{M_i}}\at{C} = 0}
\end{equation}
intersected with $\S$.

\begin{theorem}\label{thm:alg1_pathological} Using prior notation, the set of pathological set functions for \textbf{SSFT} is given by
\begin{equation}
\D_1 = 
\bigcup_{i = 0}^{n-1}\bigcup_{C \subseteq M_i: \P^{M_i}_C \neq \emptyset} \K_1(M_i, C) \cap \S, 
\end{equation}
and has Lebesgue measure zero, i.e., $\lambda(\D_1) = 0$.
\end{theorem}
%

\mypar{Complexity} By reusing queries and computations from the $(i-1)$-th iteration of \textbf{SSFT} in the $i$-th iteration, we obtain:
\begin{theorem} \textbf{SSFT} requires at most ${nk - k\log_2 k + k}$ queries and $O(n k^2)$ operations. 
\end{theorem}

\subsection{Shrinking the Set of Pathological Fourier Coefficients}

According to Theorem~\ref{thm:alg1_pathological}, the set of pathological Fourier coefficients for a given support has measure zero. However, unfortunately,
this set includes important classes of set functions including graph cuts (in the case of unit weights) and hypergraph cuts.\footnote{As an example, consider the cut function $c$ associated with the graph $V = \set{1, 2, 3}$, $E = \set{\set{1,2}, \set{2,3}}$ and $w_{12} = w_{23} = 1$, using $\widehat{\mbf{c}} = (0, 1, 2, -2, 1, 0, -2, 0)^T$. $c$ maps every subset of $V$ to the weight of the corresponding graph cut.}

%

\mypar{Solution} The key idea to exclude these and further narrow down the set of pathological cases is to use the convolution theorem \eqref{convth}, i.e., the fact that we can modulate Fourier coefficients by filtering. Concretely, we choose a random filter $h$ such that \textbf{SSFT} works for $h*s$ with probability one. $\widehat{s}$ is then obtained from $\widehat{h*s}$ by dividing by the frequency response $\overline{h}$. We keep the associated overhead in $O(n)$ by choosing a one-hop filter, i.e., $h(B)=0$ for $|B|>1$. Motivated by the fact that, e.g., the product of a Rademacher random variable (which would lead to cancellations) and a normally distributed random variable is again normally distributed, we sample our filtering coefficients i.i.d.~from a normal distribution. By filtering our signal with such a random filter we aim to end up in a situation similar to Lemma~\ref{lem:cont_random_sf}. We call the resulting algorithm \textbf{SSFT+}, shown above. 

\mypar{Algorithm} In \textbf{SSFT+} we create a random one-hop filter $h$ (lines~2--4), apply \textbf{SSFT} to the filtered signal $h*s$ (line~6) and compute $\widehat{s}$ based on $\widehat{h*s}$ (lines~8--9).

\mypar{Analysis} Building on the analysis of \textbf{SSFT}, recall that $\S$ denotes the set of $k$-Fourier-sparse (but not $(k-1)$-Fourier-sparse) set functions and
$\P_C^{M_i}$ are the elements $B \in \supp(\widehat{s})$ satisfying $B \cap M_i = C$. Let
\begin{multline}
\K_2(M_i, C) = \left\{\widehat{\mbf{s}} \in \R^k: \hsrto{s}{2^{M_i}}\at{C} = 0 \text{ and }\right. \\ \left.\hsrto{s}{2^{M_i \cup \set{x_{j}}}}\at{C} = 0 \text{ for } j \in \set{i+1, \dots, n}\right\}.
\end{multline}


\begin{theorem}\label{thm:alg2_pathological} With probability one with
    respect to the randomness of the filtering coefficients, the set
    of pathological set functions for \textbf{SSFT+} has
    the form (using prior notation)
    \begin{equation}\label{eq:D2}
        \D_2 = \bigcup_{i = 0}^{n - 2} \bigcup_{C \subseteq M_i:
          \P^{M_i}_C \neq \emptyset} \K_2(M_i, C) \cap \S. 
    \end{equation}
\end{theorem}

Theorem~\ref{thm:alg2_pathological} shows that \textbf{SSFT+} correctly processes $\hsrto{s}{2^{M_i}}\at{C} = 0$ with $\P^{M_i}_C \neq \emptyset$, iff there is an element $x \in \set{x_{i+1}, \dots, x_n}$ for which $\hsrto{s}{2^{M_i \cup \set{x}}}\at{C} \neq 0$.


\begin{theorem} If $\D_1$ is non-empty, $\D_2$ is a proper subset of $\D_1$. In particular, $\K_1(M_i, C) \cap S \neq \emptyset$ implies $\K_2(M_i, C) <_{\R} \K_1(M_i, C)$, for all $C \subseteq M_i \subseteq N$ with $\P_C^{M_i} \neq \emptyset$.
\end{theorem} 


\mypar{Complexity} There is a trade-off between the number of non-zero filtering coefficients and the size of the set of pathological set functions. 
For example, for the one-hop filters used, computing $(h*s)\at{A}$ requires $1 + n - |A|$ queries.

\begin{theorem} The query complexity of \textbf{SSFT+} is $O(n^2 k - n k \log k)$ and the algorithmic complexity is $O(n^2 k + n k^2)$. 
\end{theorem}

\section{Related Work}

We briefly discuss related work on learning set functions.

\mypar{Fourier-sparse learning} There is a substantial body of research concerned with learning Fourier/WHT-sparse set functions \cite{stobbe2012learning, scheibler2013fast, kocaoglu2014sparse, li2015active, cheraghchi2017nearly, amrollahi2019efficiently}. Recently, \cited{amrollahi2019efficiently} have imported ideas from the hashing based sparse Fourier transform algorithm~\cite{hassanieh2012nearly} to the set function setting. The resulting algorithms compute the WHT of $k$-WHT-sparse set functions with a query complexity $O(n k)$ for general frequencies, $O(k d \log n)$ for low degree ($\leq d$) frequencies and $O(kd \log n \log(d \log n))$ for low degree set functions that are only approximately sparse. To the best of our knowledge this latest work improves on previous algorithms, such as the ones by \cited{scheibler2013fast}, \cited{kocaoglu2014sparse}, \cited{li2015active}, and \cited{cheraghchi2017nearly}, providing the best guarantees in terms of both query complexity and runtime. E.g., \cited{scheibler2013fast} utilize similar ideas like hashing/aliasing to derive sparse WHT algorithms that work under random support (the frequencies are uniformly distributed on $2^N$) and random coefficient (the coefficients are samples from continuous distributions) assumptions. 
\cited{kocaoglu2014sparse} propose a method to compute the WHT of a $k$-Fourier-sparse set function that satisfies a so-called unique sign property using queries polynomial in $n$ and $2^k$. 

In a different line of work, \cited{stobbe2012learning} utilize results from compressive sensing to compute the WHT of $k$-WHT-sparse set functions, for which a superset $\mathcal{P}$ of the support is known. This approach also can be used to find a $k$-Fourier-sparse approximation and has a theoretical query complexity of $O(k \log^4 |\mathcal{P}|)$. In practice, it even seems to be more query-efficient than the hashing based WHT (see experimental section of \cited{amrollahi2019efficiently}), but suffers from the high computational complexity, which scales at least linearly with $|\mathcal{P}|$. Regrading coverage functions, to our knowledge, there has not been any work in the compressive sensing literature for the non-orthogonal Fourier bases which do not satisfy RIP properties and hence lack sparse recovery and robustness guarantees.

In summary, all prior work on Fourier-based methods for learning set functions was based on the WHT. Our work leverages the broader framework of signal processing with set functions proposed by \cited{puschel2020discrete}, which provides a larger class of Fourier transforms and thus new types of Fourier-sparsity.

\mypar{Other learning paradigms} Other lines of work for learning set functions include methods based on new neural architectures \cite{dolhansky2016deep, zaheer2017deep, weiss2017sats}, methods based on backpropagation through combinatorial solvers \cite{djolonga2017differentiable, tschiatschek2018differentiable, wang2019satnet, vlastelica2019differentiation}, kernel based methods \cite{buathong2019kernels}, and methods based on other succinct representations such as decision trees \cite{feldman2013representation} and disjunctive normal forms \cite{raskhodnikova2013learning}.

\section{Empirical Evaluation}

We evaluate the two variants of our algorithm (\textbf{SSFT} and \textbf{SSFT+}) for model 4 on three classes of real-world set functions. First, we approximate the objective functions of sensor placement tasks by Fourier-sparse functions and evaluate the quality of the resulting surrogate objective functions. Second, we learn facility locations functions (which are preference functions) that are used to determine cost-effective sensor placements in water networks \cite{leskovec2007cost}. Finally, we learn simulated bidders from a spectrum auctions test suite~\cite{weiss2017sats}.

\mypar{Benchmark learning algorithms} We compare our algorithm against three state-of-the-art algorithms for learning WHT-sparse set functions: the compressive sensing based approach \textbf{CS-WHT}~\cite{stobbe2012learning}, the hashing based approach \textbf{H-WHT}~\cite{amrollahi2019efficiently}, and the robust version of the hashing based approach \textbf{R-WHT}~\cite{amrollahi2019efficiently}. For our algorithm we set $\epsilon = 0.001$ 
and $k_{\text{max}} = 1000$. 
\textbf{CS-WHT} requires a superset $\P$ of the (unknown) Fourier support, which we set to all $B\subseteq N$ with $|B|\leq 2$ and the parameter for expected sparsity to $1000$. For \textbf{H-WHT} we used the exact algorithm without low-degree assumption and set the expected sparsity parameter to $2000$. For \textbf{R-WHT} we used the robust algorithm without low-degree assumption and set the expected sparsity parameter to $2000$ unless specified otherwise.

\subsection{Sensor Placement Tasks}

We consider a discrete set of sensors located at different fixed positions measuring a quantity of interest, e.g., temperature, amount of rainfall, or traffic data, and want to find an informative subset of sensors subject to a budget constraint on the number of sensors selected (e.g., due to hardware costs). To quantify the informativeness of subsets of sensors, we fit a multivariate normal distribution to the sensor measurements~\cite{krause2008near} and associate each subset of sensors $A \subseteq N$ with its information gain~\cite{srinivas2009gaussian}
\begin{equation}
G\at{A} = \frac{1}{2} \log |I_{|A|} + \sigma^{-2} (K_{ij})_{i, j \in A}|, 
\end{equation}
where $(K_{ij})_{i, j \in A}$ is the submatrix of the covariance matrix $K$ that is indexed by the sensors $A \subseteq N$ and $I_{|A|}$ the $|A|\times |A|$ identity matrix. We construct two covariance matrices this way for temperature measurements from 46 sensors at Intel Research \emph{Berkeley} and for velocity data from 357 sensors deployed under a highway in \emph{California}.

The information gain is a submodular set function and, thus, can be approximately maximized using the greedy algorithm by \cited{nemhauser1978analysis}: $A^* \approx \argmax_{A \subseteq N: |A| \leq d} G\at{A}$ to obtain informative subsets. We do the same using Fourier-sparse surrogates $s$ of $G$: $A^{+} \approx \argmax_{A \subseteq N: |A| \leq d} s\at{A}$ and compute $G(A^+)$. As a baseline we place $d$ sensors at random $A_{\text{rand}}$ and compute $G\at{A_{\text{rand}}}$. Figure~\ref{fig:information_gain} shows our results. The x-axes correspond to the cardinality constraint used during maximization and the y-axes to the information gain obtained by the respective informative subsets. In addition, we report next to the legend the execution time and number of queries needed by the successful experiments.

\begin{figure*}[]
    \centering
    \hfill
    \begin{subfigure}[c]{0.48\textwidth}
    \centering
    \includegraphics[width=0.64\textwidth]{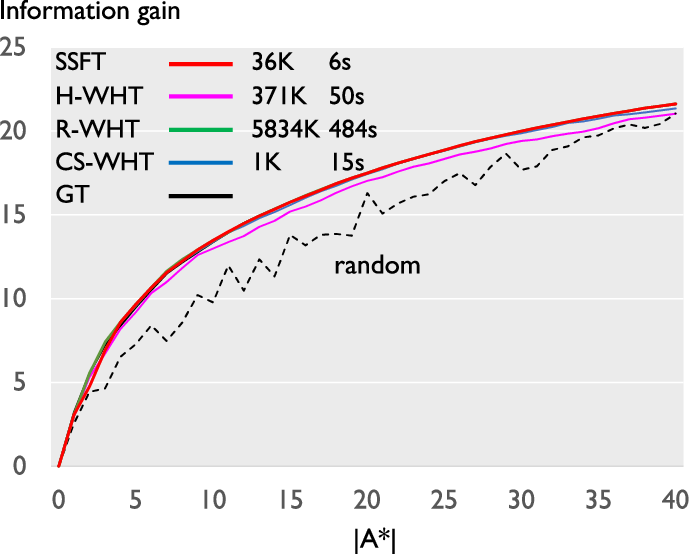}
    \subcaption{Berkeley, $n = 46$}
    \end{subfigure}
    \hfill
    \begin{subfigure}[c]{0.48\textwidth}
    \centering
    \includegraphics[width=0.64\textwidth]{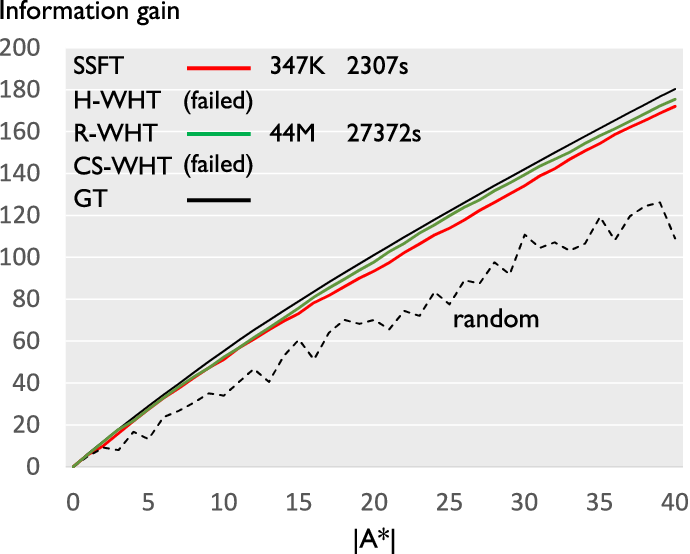}
    \subcaption{California, $n = 357$}
    \end{subfigure}
    \hfill
    \caption{Comparison of learnt surrogate objective functions on submodular maximization tasks subject to cardinality constraints (x-axis); On the y-axis we plot the information gain achieved by the informative subset obtained by the respective method. We report the number of queries and execution time in seconds next to the legend or indicate failure.}\label{fig:information_gain}
\end{figure*}

\mypar{Interpretation of results} \textbf{H-WHT} only works for the \emph{Berkeley} data. For the other data set it is not able to reconstruct enough Fourier coefficients to provide a meaningful result. The likely reason is that the target set function is not exactly Fourier-sparse, which can cause an excessive amount of collisions in the hashing step. In contrast, \textbf{CS-WHT} is noise-robust and yields sensor placements that are indistinguishable from the ones obtained by maximizing the true objective function in the first task. However, for the \emph{California} data, \textbf{CS-WHT} times out. In contrast, \textbf{SSFT} and \textbf{R-WHT} work well on both tasks. In the first task, \textbf{SSFT} is on par with \textbf{CS-WHT} in terms of sensor placement quality and significantly faster despite requiring more queries. On the \emph{California} data, \textbf{SSFT} yields sensor placements of similar quality as the ones obtained by \textbf{R-WHT} while requiring orders of magnitude fewer queries and time.

\subsection{Learning Preference Functions}

We now consider a class of preference functions that are used for the cost-effective contamination detection in water networks \cite{leskovec2007cost}. The networks stem from the Battle of Water Sensor Networks (BSWN) challenge \cite{ostfeld2008battle}. The junctions and pipes of each BSWN network define a graph. Additionally, each BSWN network has dynamic parameters such as time-varying water consumption demand patterns, opening and closing valves, and so on. 

To determine a cost-effective subset of sensors (e.g., given a maximum budget), \cited{leskovec2007cost} make use of facility locations functions of the form
\begin{equation}\label{eq:fl}
    p: 2^N \to \mathbb{R}; A \mapsto \sum_{\ell=1}^L \max_{i \in A} r_{\ell i},
\end{equation}
where $r$ is a matrix in $\R_{\geq 0}^{L \times n}$. Each row corresponds to an event (e.g., contamination of the water network at any junction) and the entry $r_{\ell i}$ quantifies the utility of the $i$-th sensor in case of the $\ell$-th event. It is straightforward to see that \eqref{eq:fl} is a preference function with $a = 0$ and $u_i = \sum_{\ell=1}^L r_{\ell i}$. Thus, they are sparse w.r.t.~model~4 and dense w.r.t.~WHT (see Lemma~\ref{lem:fldc}). 

\cited{leskovec2007cost} determined three different utility matrices $r \in \R^{3424 \times 12527}$ that take into account the fraction of events detected, the detection time, and the population affected, respectively. The matrices were obtained by costly simulating millions of possible contamination events in a 48 hour timeframe. For our experiments we select one of the utility matrices and obtain subnetworks by selecting the columns that provide the maximum utility, i.e., we select the $|N| = n$ columns $j$ with the largest $\max_{\ell} r_{\ell j}$. 

\begin{table}
\scriptsize
\centering
\caption{Comparison of model~4 sparsity (\textbf{SSFT}) against WHT sparsity (\textbf{R-WHT}) of facility locations functions in terms of reconstruction error $\|\mbf{p} - \mbf{p'}\|/\|\mbf{p}\|$ for varying $|N|$; The italic results are averages over 10 runs.}\label{tab:fln}
    \begin{tabular}{@{}rlrrrrr@{}}\toprule
$|N|$ &  & $\alpha$ & queries & time (s) & $k$  & $\|\mbf{p} - \mbf{p'}\|/\|\mbf{p}\|$ \\\midrule
$20$ & \textbf{SSFT} & & \boldmath$734$ & \boldmath$0.12$ & \boldmath$102$ & \boldmath$0$ \\
& WHT & & & &  $2^{9}$ & $0.000143$  \\
& & & & & $2^{14}$ & $0.000078$  \\
& & & & & $2^{19}$ & $0.000001$ \\\midrule
$50$ &  \textbf{SSFT}& & \boldmath$10K$ & \boldmath$2$ & \boldmath$648$ & \boldmath$0$ \\
& R-WHT & 1 & $\mathit{2,103K}$ & $\mathit{361}$ & $\mathit{1,380}$ &  $\mathit{0.001744}$ \\
&  & 2 & $\mathit{4,192K}$ & $\mathit{766}$ & $\mathit{2,739}$ &  $\mathit{0.000847}$ \\
&  & 4 & $\mathit{8,370K}$ & $\mathit{1,499}$ & $\mathit{5,054}$ &  $\mathit{0.000129}$ \\
&  & 8 & $\mathit{16,742K}$ & $\mathit{2,838}$ & $\mathit{9,547} $ &  $\mathit{0.000108}$ \\
\midrule

$100$ & \textbf{SSFT} & & \boldmath$76K$ & \boldmath$24$ & \boldmath$2,308$ &  \boldmath$0$ \\
& R-WHT & 1 & $16,544K$ & $5,014$ & $2,997$ &  $0.000546$ \\
&  & 2 & $33,100K$ & $10,265$ & $6,466$ &  $0.000380$ \\\midrule
$200$ & \textbf{SSFT} & & \boldmath$494K$ & \boldmath$451$ & \boldmath$7,038$ &  \boldmath$0$ \\
$300$ & \textbf{SSFT} & & \boldmath$1,644K$ & \boldmath$2,368$ & \boldmath$16,979$ & \boldmath$0$ \\
$400$ & \textbf{SSFT} & & \boldmath$3,859K$ & \boldmath$7,654$ & \boldmath$28,121$ & \boldmath$0$ \\
$500$ & \textbf{SSFT} & & \boldmath$7,218K$ & \boldmath$17,693$ & \boldmath$38,471$ &  \boldmath$0$ \\
\bottomrule
\end{tabular}
\end{table}

\begin{table*}[]\caption{Multi-region valuation model ($n=98$). Each row corresponds to a different bidder type.}
  \label{tab:mrvm}
\resizebox{\textwidth}{!}{
\begin{tabular}{@{}lrrr|rrr|rrr@{}}
    \toprule
  &  \multicolumn{3}{c|}{number of queries (in thousands)} & \multicolumn{3}{c|}{Fourier coefficients recovered}  & \multicolumn{3}{c}{relative reconstruction error} \\
B. type & \textbf{SSFT} & \textbf{SSFT+} & \textbf{H-WHT} & \textbf{SSFT} & \textbf{SSFT+} & \textbf{H-WHT} & \textbf{SSFT} & \textbf{SSFT+} & \textbf{H-WHT}\\\midrule
local & $3 \pm 4$ & $229 \pm 73$ & $781 \pm 0$ & $118 \pm 140$&$303 \pm 93$&$675 \pm 189$&$0.5657 \pm 0.4900$&$0 \pm 0$&$0 \pm 0$\\
regional & $20 \pm 1$ & $646 \pm 12$ & $781 \pm 0$ & $659 \pm 32$&$813 \pm 36$&$1,779 \pm 0$&$0.0118 \pm 0.0071$&$0 \pm 0$&$0 \pm 0$\\
national & $71 \pm 0$ & $3,305 \pm 1$ & $781 \pm 0$ & $1,028 \pm 3$&$1,027 \pm 6$&$4,170 \pm 136$&$0.0123 \pm 0.0014$&$0.0149 \pm 0.0089$&$0.2681 \pm 0.2116$\\

    \bottomrule
\end{tabular}
}
\end{table*}

In Table~\ref{tab:fln} we compare the sparsity of the corresponding facility locations function in model~4 against its sparsity in the WHT. For $|N| = 20$, we compute the full WHT and select the $k$ largest coefficients. For $|N| > 20$, we compute the $k$ largest WHT coefficients using \textbf{R-WHT}. The model~4 coefficients are always computed using \textbf{SSFT}. If the facility locations function is $k$ sparse w.r.t.~model~4 for some $|N| = n$, we set the expected sparsity parameter of \textbf{R-WHT} to different multiples $\alpha k$ up to the first $\alpha$ for which the algorithm runs out of memory.  We report the number of queries, time, number of Fourier coefficients $k$, and relative reconstruction error. For \textbf{R-WHT} experiments that require less than one hour we report average results over 10 runs (indicated by italic font).
For $|N| > 20$, the relative error cannot be computed exactly and thus is obtained by sampling 100,000 sets $\mathcal{A}$ uniformly at random and computing $\|\mbf{p}_{\mathcal{A}} - \mbf{p'}_{\mathcal{A}}\|/\|\mbf{p}_{\mathcal{A}}\|$, where $p$ denotes the real facility locations function and $p'$ the estimate.

\mypar{Interpretation of results} The considered facility locations functions are indeed sparse w.r.t.~model~4 and dense w.r.t. the WHT. As expected, \textbf{SSFT} outperforms \textbf{R-WHT} in this scenario, which can be seen by the lower number of queries, reduced time, and an error of exactly zero for the SSFT. This experiment shows certain classes of set functions of practical relevance are better represented in the model~4~basis than in the WHT~basis.

\subsection{Preference Elicitation in Auctions}

In combinatorial auctions a set of goods $N = \set{x_1, \dots, x_n}$ is auctioned to a set of $m$ bidders. Each bidder $j$ is modeled as a set function $b_j: 2^N \to \R$ that maps each bundle of goods to its subjective value for this bidder. The problem of learning bidder valuation functions from queries is known as the preference elicitation problem \cite{brero2019machine}. Our experiment sketches an approach under the assumption of Fourier sparsity. 

As common in this field~\cite{weissteiner2019deep, weissteiner2020fourier}, we resort to simulated bidders. Specifically, we use the multi-region valuation model (MRVM) from the spectrum auctions test suite \cite{weiss2017sats}. In MRVM, 98 goods are auctioned off to 10 bidders of different types (3 local, 4 regional, and 3 national). We learn these bidders using the prior Fourier-sparse learning algorithms, this time including \textbf{SSFT+}, but excluding \textbf{CS-WHT}, since $\P$ is not known in this scenario. 
Table~\ref{tab:mrvm} shows the results: means and standard deviations of the number of queries required, Fourier coefficients recovered, and relative error (estimated using 10,000 samples) taken over the bidder types and 25 runs.

\mypar{Interpretation of results} First, we note that \textbf{SSFT+} can indeed improve over \textbf{SSFT} for set functions that are relevant in practice. Namely, \textbf{SSFT+} consistently learns sparse representations for local and regional bidders, while \textbf{SSFT} fails. \textbf{H-WHT} also achieves perfect reconstruction for local and regional bidders. For the remaining bidders none of the methods achieves perfect reconstruction, which indicates that those bidders do not admit a sparse representation. Second, we observe that, for the local and regional bidders, in the non-orthogonal model~4 basis only half as many coefficients are required as in the WHT basis. Third, \textbf{SSFT+} requires less queries than \textbf{H-WHT} in the Fourier-sparse cases.  

\section{Conclusion} 

We introduced an algorithm for learning set functions that are sparse with respect to various generalized, non-orthogonal Fourier bases. In doing so, our work significantly expands the set of efficiently learnable set functions. As we explained, the new notions of sparsity connect well with preference functions in recommender systems and the notions of complementarity and substitutability, which we consider an exciting avenue for future research.


\section*{Ethics Statement}
Our approach is motivated by a range of real world applications, including modeling preferences in recommender systems and combinatorial auctions, that require the modeling, processing, and analysis of set functions, which is notoriously difficult due to their exponential size. Our work adds to the tool set that makes working with set functions computationally tractable. 
Since the work is of foundational and algorithmic nature we do not see any immediate ethical concerns. In case that the models estimated with our algorithms are used for making decisions (such as recommendations, or allocations in combinatorial auctions), of course additional care has to be taken to ensure that ethical requirements such as fairness are met. These questions are complementary to our work. 

\bibliography{refs}
\appendix
\section*{Appendix}
\section{Preference Functions}

\begin{table*}[]
\centering
	\caption{Shifts and Fourier concepts.\label{tab:sfdsp_appendix}}
	\tiny
	$
	\begin{array}{@{}l llllll@{}}\toprule
	\text{model } & \text{shift } T_Q s\at{A} & B\text{-th basis vec. } \mbf{f}^B_A =  & F \text{ (sum)}: \widehat{s}\at{B} = & F^{-1} \text{ (sum)}: {s}_A = &  \text{known as} \\ \midrule

	 3 & s\at{A \setminus Q} & (-1)^{|B|}\iota_{B \subseteq A} & \dps\sum_{\stackrel{A\subseteq B}{\hphantom{ A \cup B = N}}} (-1)^{|A|}s\at{A} & 
	\hphantom{\tfrac{1}{2^{n}}} \dps\sum_{\stackrel{B\subseteq A}{\hphantom{A \cap B = \emptyset}}} (-1)^{|B|}\widehat{s}\at{B} &
	 
	\text{Zeta transform \cite{bjorklund2007fourier}}
	\\[5mm]
	4 & s\at{A \cup Q} & \iota_{A \cap B = \emptyset} & \dps\sum_{\stackrel{A\subseteq N:}{ A \cup B = N}} (-1)^{|A \cap B|}s\at{A} & 
	\hphantom{\tfrac{1}{2^{n}}} \dps\sum_{\stackrel{B\subseteq N:}{A \cap B = \emptyset}} \widehat{s}\at{B} &
	
	\text{W-transform \cite{chakrabarty2012testing}}  \\[5mm]
	5 & s\at{A \setminus Q \cup Q \setminus A} & \tfrac{1}{2^{n}}(-1)^{|A \cap B|} & \dps\sum_{\stackrel{A\subseteq N}{\hphantom{ A \cup B = N}}} (-1)^{|A \cap B|} s\at{A} & 
	\tfrac{1}{2^{n}}\dps\sum_{\stackrel{B\subseteq N}{\hphantom{A \cap B = \emptyset}}} (-1)^{|A \cap B|}\widehat{s}\at{B} &
	
	\text{WHT \cite{bernasconi1996fourier}}  \\\bottomrule

	\end{array} 
	$
	
\end{table*}

Let $N = \set{x_1, \dots, x_n}$ denote our ground set. For this section, we assume $x_1 = 1, \dots, x_n = n$. 

An important aspect of our work is that certain set functions are sparse in one basis but not in the others. In this section we show that {\em preference functions}~\cite{djolonga2016variational} indeed constitute a class of set functions that are sparse w.r.t.~model~4 (see Table~\ref{tab:sfdsp_appendix}, which we replicate from the paper for convenience) and dense w.r.t.~model~5 (= WHT basis). Preference functions naturally occur in machine learning tasks on discrete domains such as recommender systems and auctions, in which they, e.g., are used to model complementary- and substitution effects between goods. Goods complement each other when their combined utility is greater than the sum of their individual utilities. E.g., a pair of shoes is more useful than the two shoes individually and a round trip has higher utility than the combined individual utilities of outward and inward flight. Analogously, goods substitute each other when their combined utility is smaller than the sum of their individual utilities. E.g., it might not be necessary to buy a pair of glasses if you already have one. Formally, a preference function is given by
\begin{linenomath*}\begin{equation}\label{eq:fldc_appendix}
\begin{aligned}
p: 2^N \to \mathbb{R}; A \mapsto \sum_{i \in A} u_i + \sum_{\ell=1}^L \left(\max_{i \in A} r_{\ell i} - \sum_{i \in A} r_{\ell i}\right) \\- \sum_{k=1}^K \left(\max_{i \in A} a_{ki} - \sum_{i \in A} a_{ki}\right).
\end{aligned}
\end{equation}\end{linenomath*}
Equation~\eqref{eq:fldc_appendix} is composed of a modular part parametrized by $u \in \R^n$, a repulsive part parametrized by $r \in \R_{\geq 0}^{L \times n}$, with ${L \in \mathbb{N}}$, and an attractive part parametrized by $a \in \R_{\geq 0}^{K \times n}$, with $K \in \mathbb{N}$.  

\begin{lemma}\label{lem:fldc_appendix} Preference functions of the form \eqref{eq:fldc_appendix} are Fourier-sparse w.r.t.~model~4.
\end{lemma}
\begin{proof} In order to prove that preference functions are sparse w.r.t.~model~4 we exploit the linearity of the Fourier transform. That is, we are going to show that $p$ is Fourier sparse by showing that it is a sum of Fourier sparse set functions. In particular, there are only two types of summands (= set functions): 

First, $A \mapsto \sum_{i \in A} u_i$, $A \mapsto - \sum_{i \in A} r_{\ell i}$, for $\ell \in \set{1, \dots, L}$, and $A \mapsto \sum_{i \in A} a_{ki}$, for $k \in \set{1, \dots, K}$, are modular set functions whose only non-zero Fourier coefficients are summed up in $\widehat{p}\at{\set{x}}$ for $x \in N$ and $\widehat{p}\at{\emptyset}$. 

Second, $f_{\ell}(A) = \max_{i \in A} r_{{\ell}i}$, for $\ell \in \set{1, \dots, L}$, and $g_k(A) = - \max_{i \in A} a_{ki}$, for $k \in \set{1, \dots, K}$, are weighted- and  negative weighted coverage functions, respectively. In order to see that ${f_{\ell}\at{A} = \max_{i \in A} r_{{\ell}i}}$ is a weighted coverage function, observe that the codomain of $f_{\ell}$ is $\set{r_{{\ell}1}, \dots, r_{{\ell}n}}$. Let $\sigma: N \to N$ denote the permutation that sorts $r_{\ell}$, i.e., $r_{{\ell}\sigma(1)} < r_{{\ell}\sigma(2)} < \cdots < r_{{\ell}\sigma(n)}$. Let $U = \set{1, \dots, n}$ denote the universe. We set ${w(u) = r_{{\ell}\sigma(u)} - r_{{\ell}\sigma(u - 1)}}$ for $u \geq 2$ and ${w(u) = r_{{\ell}\sigma(1)}}$ for $u = 1$. Let $M = \set{S_1, \dots, S_n}$. Let the set $S_i = \set{1, \dots, \sigma^{-1}(i)}$. In words, $S_i$ covers all elements $u$ that correspond to smaller values $r_{\ell \sigma(u)}$, i.e., $u$ with  $r_{\ell \sigma(u)} \leq r_{\ell i}$. By construction of $w(u)$ we have $\mbf{w}(S_i) = r_{\ell \sigma(1)} + (r_{\ell \sigma(2)} - r_{\ell \sigma(1)}) + \dots + (r_{\ell \sigma(\sigma^{-1}(i))} - r_{\ell \sigma(\sigma^{-1}(i) - 1)}) = r_{\ell i}$, and, because of  $S_{\sigma(1)} \subset S_{\sigma(2)} \subset \cdots \subset S_{\sigma(n)}$ we have, for all $A \subseteq N$, 
\begin{linenomath*}\begin{equation}
\mbf{w}\left(\bigcup_{i \in A} S_i \right) = \mbf{w}(S_j) = r_{\ell j},
\end{equation}\end{linenomath*}
where $j$ is the element in $A$ that satisfies $S_i \subseteq S_j$ for all $i \in A$. Now, observe that $S_i \subseteq S_j$ is equivalent to $r_{\ell i} \leq r_{\ell j}$. Thus, by definition of $\max$ we have ${r_{\ell j} = \max_{i \in A} r_{\ell i}}$. 

The same construction works for $g_k(A) = - \max_{i \in A} a_{ki}$. Weighted coverage functions with $|U| = n$ are $n$-Fourier-sparse with respect to the W-transform~\cite{chakrabarty2012testing} and $(n + 1)$-Fourier-sparse with respect to model~4 (one additional coefficient for $\emptyset$). The preference function $p$ is a sum of $1 + K + L$ modular set functions, $K$ sparse weighted coverage functions that require at most $n$ additional Fourier coefficients (with $B \neq \emptyset$) each and $L$ sparse negative weighted coverage functions that require at most $n$ additional Fourier coefficients each. Therefore, $p$ has at most $1+ n + L n + K n = O(L n + K n)$ non-zero Fourier coefficients w.r.t.~model~4.
\end{proof}

\begin{remark} The construction in the second part of the proof of Lemma~\ref{lem:fldc_appendix} shows that preference functions with $L + K \geq 1$ are dense w.r.t.~the WHT basis, because there is an element in $U$ that is covered by all $S_1, \dots, S_n$.
\end{remark}

\section{SSFT: Support Discovery}

In this section we prove the equations necessary for the support discovery mechanism of $\textbf{SSFT}$. 

Let $s: 2^N \to \mathbb{R}$ be a set function and let $M_i = \set{x_1, \dots, x_i} \subseteq N$. As before we denote the restriction of $s$ to $M_i$ with
\begin{linenomath*}\begin{equation}\label{eq:restriction_simple}
\srto{s}{2^{M_i}}: 2^{M_i} \to \mathbb{R}; A \mapsto s\at{A}.
\end{equation}\end{linenomath*}

Recall the problem we want to solve and our algorithms (Fig.~\ref{fig:algorithms}) for doing so (under mild assumptions on the Fourier coefficients).

\begin{problem}[Sparse Fourier transform]\label{prob:general_appendix} Given oracle access to query a $k$-Fourier-sparse set function $s$, compute its Fourier support and associated Fourier coefficients.
\end{problem}

\begin{figure*}[]
    \begin{minipage}[t]{0.5\textwidth}
        \null
\begin{ssft}[H]
\renewcommand{\thealgorithm}{}
\caption{Sparse set function Fourier transform of $s$}\label{alg:sdsft4_appendix}
\begin{algorithmic}[1]
	\State{$M_0 \gets \emptyset$}
	\State{$\hsrto{s}{2^{M_0}}\at{\emptyset} \gets s\at{\emptyset}$}
	\For{$i = 1, \dots, n$}
	\State{$M_i \gets M_{i-1} \cup \set{x_i}$}
	\State{$\mathcal{B} \gets \emptyset, \mathcal{A} \gets \emptyset$}
	\For{$B \in \supp(\hsrto{s}{2^{M_{i-1}}})$}
	\State{$\mathcal{B} \gets \mathcal{B} \cup \set{B, B \cup \set{x_i}}$}
	\State{$\mathcal{A} \gets \mathcal{A} \cup \set{M_i \setminus
            B, M_i \setminus (B \cup \set{x_i})}$}
	\EndFor
	\State{$\mbf{s}_{\mathcal{A}} \gets (s\at{A})_{A \in \mathcal{A}}$}
	\State{$\mbf{x} \gets \text{ solve } \mbf{s}_{\mathcal{A}} = F^{-1}_{\mathcal{A}\mathcal{B}} \mbf{x}$ for $\mbf{x}$}
	\For{$B \in \mathcal{B}$ with $\mbf{x}_B \neq 0$}
	\State $\hsrto{s}{2^{M_i}}\at{B} \gets \mbf{x}_B$
	\EndFor
	\EndFor
	\State \Return{$\hsrto{s}{2^{M_n}}$}
    \end{algorithmic}
\end{ssft}
\end{minipage}
\hfill
\begin{minipage}[t]{0.47\textwidth}
    \null
\begin{ssft2}[H]
\renewcommand{\thealgorithm}{}
\caption{Filtering based \textbf{SSFT} of $s$}\label{alg:sdsft4_appendix_general}
\begin{algorithmic}[1]
	\State{\em // Sample random coefficients.}
	\State{$h\at{\emptyset} = 1$
          \vphantom{{$\hsrto{s}{2^{M_0}}\at{\emptyset} \gets
              s\at{\emptyset}$}}} 
	\For{$x \in \{x_1, \dots, x_n\}$}
	\State{$h\at{\set{x}} \gets c \sim \mathcal{N}(0, 1)$}
	\EndFor
	\State{\em // Fourier transform of filtered set function.}
	\State{$\widehat{h * s} \gets $\textbf{SSFT}($h *
          s$)}
	\State{\em // Compute the original coefficients.}
	\For{$B \in \supp(\widehat{h * s})$}
	\State{$\widehat{s}_B \gets {\widehat{(h * s)}\at{B}}/{\overline{h}\at{B}}$}
	\EndFor	
	\State \Return{$\widehat{s}$}
        \Statex
        \Statex
        \Statex
    \end{algorithmic}
    \vspace{-0.2pt}
\end{ssft2}
\end{minipage}
\caption{\textbf{SSFT} and \textbf{SSFT+} for model~4.}\label{fig:algorithms}
\end{figure*}

\begin{lemma}[Model~4]\label{thm:hashing_model4} Using prior notation we have
\begin{linenomath*}\begin{equation}\label{eq:hashing_model4}
\hsrto{s}{2^{M_i}}\at{B} = \sum_{A \subseteq N \setminus M_i} \widehat{s}\at{A \cup B}.
\end{equation}\end{linenomath*}
\end{lemma}
\begin{proof}
We have $\srto{s}{2^{M_i}}\at{C} = s\at{C}$ per definition, for all $C \in 2^{M_i}$. Performing the Fourier expansion on both sides yields
\begin{linenomath*}\begin{equation}\label{eq:hashing_model4_ls}
\begin{aligned}
\sum_{B \subseteq M_i \setminus C} \hsrto{s}{2^{M_i}}\at{B} &= \sum_{B \subseteq N \setminus C} \widehat{s}\at{B} \\
&= \sum_{B \subseteq M_i \setminus C} \sum_{A \subseteq N \setminus M_i} \widehat{s}\at{A \cup B}.
\end{aligned}
\end{equation}\end{linenomath*}
\eqref{eq:hashing_model4} is the unique solution for the system of $2^i$ equations given by \eqref{eq:hashing_model4_ls}.
\end{proof}

\begin{corollary}[Model~4]\label{cor:sparsity_propagation_model4} Let $x \in N \setminus M_i$. Using prior notation we have
\begin{linenomath*}\begin{equation}
\begin{aligned}\label{eq:sparsity_propagation_model4}
\hsrto{s}{2^{M_i}}\at{B} &= \hsrto{s}{2^{M_i \cup \set{x}}}\at{B} + \hsrto{s}{2^{M_i \cup \set{x}}}\at{B \cup \set{x}}.
\end{aligned}
\end{equation}\end{linenomath*} 
\end{corollary}
\begin{proof}
The claim follows from the simple derivation
\begin{linenomath*}\begin{equation}
\begin{aligned}\label{eq:sparsity_propagation_model4_proof}
\hsrto{s}{2^{M_i}}\at{B} &= \sum_{A \subseteq N \setminus M_i} \widehat{s}\at{A \cup B}\\
&= \sum_{A \subseteq N \setminus (M_i \cup \set{x})} \widehat{s}\at{A \cup B} \\ & \ + \sum_{A \subseteq N \setminus (M_i \cup \set{x})} \widehat{s}\at{A \cup B \cup \set{x}}\\
&= \hsrto{s}{2^{M_i \cup \set{x}}}\at{B} + \hsrto{s}{2^{M_i \cup \set{x}}}\at{B \cup \set{x}}.
\end{aligned}
\end{equation}\end{linenomath*} 
\end{proof}

Corollary~\ref{cor:sparsity_propagation_model4} is used in lines~6-8 of \textbf{SSFT} in Fig.~\ref{fig:algorithms} to find a superset of the Fourier coefficients of $\hsrto{s}{2^{M_i}}$ using $\supp(\hsrto{s}{2^{M_{i-1}}})$.

\section{SSFT: Pathological Examples}

In this section we provide proofs and derivations of the sets of pathological set functions $\D_1$ (for \textbf{SSFT}) and $\D_2$ (for \textbf{SSFT+}). In order to do so, we consider set functions $s$ that are ${k\text{-Fourier-sparse}}$ (but not $(k-1)$-Fourier-sparse) with  support ${\supp(\widehat{s}) = \set{B_1, \dots, B_k} = \mathcal{B}}$, i.e., 
${\set{s: 2^N \to \mathbb{R}: \widehat{s}\at{B} \neq 0 \text{ iff } B
  \in \mathcal{B}}}$, which is isomorphic to
\begin{linenomath*}\begin{equation}
    \mathcal{S} =
\set{\widehat{\mbf{s}} \in \mathbb{R}^k: \widehat{\mbf{s}}_i \neq 0 \text{ for all } i \in \set{1, \dots, k}}.
\end{equation}\end{linenomath*}
Here, and in the following, we identify $i \in \set{1, \dots, k}$ with $B_i$ and, in particular, $\widehat{\mbf{s}}_i$ with $\widehat{s}\at{B_i}$. We denote the Lebesgue measure on $\mathbb{R}^k$ with $\lambda$ and set $\P^{M_i}_C \mydef {\set{B \in \B: B \cap M_i = C}}$. We start with the analysis of \textbf{SSFT}.

\subsection{Pathological Set Functions for SSFT}

\textbf{SSFT} fails to compute the Fourier coefficients for which $\hsrto{s}{2^{M_i}}\at{C} = 0$ despite $\P^{M_i}_C \neq \emptyset$. Thus, the set of pathological set functions $\D_1$ can be written as the finite union of kernels of the form
\begin{linenomath*}\begin{equation}
\K_1(M_i, C) = \set{\widehat{\mbf{s}} \in \R^k: \hsrto{s}{2^{M_i}}\at{C} = 0} 
\end{equation}\end{linenomath*}
intersected with $\S$.

\begin{theorem}\label{thm:alg1_pathological_appendix} Using prior notation, the set of pathological set functions for \textbf{SSFT} is given by
\begin{linenomath*}\begin{equation}
\begin{aligned}
\D_1 &\mydef \set{\widehat{\mbf{s}} \in \S: \text{ \textbf{SSFT} fails}}\\ &= \bigcup_{i = 0}^{n-1}\bigcup_{C \subseteq M_i: \P^{M_i}_C \neq \emptyset} \K_1(M_i, C) \cap \S, 
\end{aligned}
\end{equation}\end{linenomath*}
and has Lebesgue measure zero, i.e., $\lambda(\D_1) = 0$.
\end{theorem}

For the proof of Theorem~\ref{thm:alg1_pathological_appendix} the following Lemma is useful:
\begin{lemma}\label{lem:support_measure} Let $A \subseteq \mathbb{R}^k$ be a $\lambda$-measurable set. We have $\lambda(A \cap \S) = \lambda(A)$.
\end{lemma}
\begin{proof}
First, we observe that the complement of $\S$ has Lebesgue measure zero, as it is a finite union of $(k-1)$-dimensional linear subspaces of $\R^k$:
\begin{linenomath*}\begin{equation}\label{eq:S_complement}
\begin{aligned}
\S^c &= \set{\widehat{\mbf{s}} \in \mathbb{R}^k: \widehat{\mbf{s}}_i = 0 \text{ for some } i \in \set{1, \dots, k}} \\
&= \bigcup_{i = 1}^k \ker(\diag \iota_{i}),
\end{aligned}
\end{equation}\end{linenomath*}
in which $\diag \iota_i$ is the $k \times k$ diagonal matrix with a one at position $i$ and zeros everywhere else. 

Now, the claim follows by simply writing $A$ as disjoint union $(A \cap \S) \dcup (A \cap \S^c)$:
\begin{linenomath*}\begin{equation}
\lambda(A) = \lambda(A \cap \S) + \lambda(A \cap \S^c) = \lambda(A \cap \S).
\end{equation}\end{linenomath*}
\end{proof}

\begin{proof}[Proof of Theorem~\ref{thm:alg1_pathological_appendix}]
\textbf{SSFT} fails iff there are cancellations of Fourier coefficients in one of the processing steps. That is, iff
\begin{linenomath*}\begin{equation}
\hsrto{s}{2^{M_i}}\at{C} = \sum_{B \in \P^{M_i}_C} \widehat{s}\at{B} = 0,
\end{equation}\end{linenomath*}
despite $\P^{M_i}_C \neq \emptyset$. Note that the set of set functions that fail in processing step $M_i$ at $C \subseteq M_i$ can be written as the intersection of a $(k-1)$-dimensional linear subspace $\K_1(M_i, C)$ and $\S$. To obtain $\D_1$, we collect all such sets:
\begin{linenomath*}\begin{equation}
\begin{aligned}
\D_1 &= \bigcup_{i = 0}^{n - 1} \bigcup_{C \subseteq M_i: \P^{M_i}_C \neq \emptyset} \K_1(M_i, C) \cap \S.
\end{aligned}
\end{equation}\end{linenomath*} 
Now, the claim follows from
\begin{linenomath*}\begin{equation}
\begin{aligned}
\lambda(\D_1) &\leq \sum_{i = 0}^{n-1} \sum_{C \subseteq M_i: \P^{M_i}_C \neq \emptyset} \underbrace{\lambda(\K_1(M_i, C) \cap \S)}_{= 0} = 0,
\end{aligned}
\end{equation}\end{linenomath*}
by Lemma~\ref{lem:support_measure} and the fact that proper subspaces have measure zero.
\end{proof}

Consequently, the set of $k$-Fourier-sparse (but not $(k-1)$-Fourier-sparse) set functions for which $\textbf{SSFT}$ works is $\S \setminus \D_1$. Lemma~\ref{lem:cont_random_sf} characterizes an important family of set functions in $\S \setminus \D_1$.

\begin{lemma}\label{lem:cont_random_sf} \textbf{SSFT} correctly computes the Fourier transform of Fourier-sparse set functions $s$ with $\supp(\widehat{s}) = \mathcal{B}$ and randomly sampled Fourier coefficients, that satisfy
\begin{enumerate}
\item $\widehat{s}\at{B} \sim P_B$, where $P_B$ is a continuous probability distribution with density function $p_B$, for $B \in \supp(\widehat{s})$,
\item $p_{\mathcal{B}} = \prod_{B \in \mathcal{B}} p_B$,
\end{enumerate} 
with probability one.
\end{lemma}
\begin{proof}
The sum of independent continuous random variables is a continuous random variable. Thus, for $\P^{M_i}_C \neq \emptyset$, the event $\sum_{B \in \P^{M_i}_C} \widehat{s}\at{B} = 0$ has probability zero. Equivalently, we have 
\begin{linenomath*}\begin{equation}
\prob{\dps\sum_{B \in \P^{M_i}_C} \widehat{s}\at{B} \neq 0} = 1
\end{equation}\end{linenomath*}

for all $i \in \{0, \dots, n\}$ and $C \subseteq M_i$ with $\P^{M_i}_C \neq \emptyset$.
\end{proof}

\subsection{Shrinking the Set of Pathological Fourier Coefficients (SSFT+)}

We now analyze the set of pathological Fourier coefficients for \textbf{SSFT+}. Building on the analysis of \textbf{SSFT}, recall that $\S$ denotes the set of $k$-Fourier-sparse (but not $(k-1)$-Fourier-sparse) set functions and
$\P_C^{M_i}$ are the elements $B \in \supp(\widehat{s})$ satisfying $B \cap M_i = C$. Let
\begin{linenomath*}\begin{equation}
\begin{aligned}
\K_2(M_i, C) &\mydef \set{\widehat{\mbf{s}} \in \R^k: \hsrto{s}{2^{M_i}}\at{C} = 0 \text{ and } \\&\quad\quad \hsrto{s}{2^{M_i \cup \set{x_{j}}}}\at{C} = 0 \text{ for all } \\ &\quad\quad j \in \set{i+1, \dots, n}}.
\end{aligned}
\end{equation}\end{linenomath*}

As we use a random one-hop filter in \textbf{SSFT+}, the randomness of the filtering coefficients needs to be taken into account.

\begin{theorem}\label{thm:alg2_pathological_appendix} With probability one with
    respect to the randomness of the filtering coefficients, the set
    of pathological set functions for \textbf{SSFT+} has
    the form (using prior notation)
    \begin{linenomath*}\begin{equation}\label{eq:D2_appendix}
        \D_2 = \bigcup_{i = 0}^{n - 2} \bigcup_{C \subseteq M_i:
          \P^{M_i}_C \neq \emptyset} \K_2(M_i, C) \cap \S. 
    \end{equation}\end{linenomath*}
\end{theorem}
\begin{proof} 
We fix $M_i$ and a $C$ with $\P^{M_i}_C \neq \emptyset$. By applying definitions and rearranging sums we obtain 
\begin{linenomath*}\begin{equation}
\scriptsize
\begin{aligned}
\hsrto{(h*s)}{2^{M_i}}(C) &= \hsrto{s}{2^{M_i}}\at{C} +  \sum_{x \in M_i \setminus C} h\at{\set{x}}\hsrto{s}{2^{M_i}}\at{C} \\& \ + \sum_{x \in N \setminus M_i} h\at{\set{x}}\hsrto{s}{2^{M_i \cup \set{x}}}\at{C}.
\end{aligned}
\end{equation}\end{linenomath*}
Each filtering coefficient $h\at{\set{x}}$ is the realization of a random variable $H_x \sim \mathcal{N}(0, 1)$. Thus, the probability of failure (for $M_i$ and $C$ fixed) can be written as 
\begin{linenomath*}\begin{equation}\label{eq:prob_failure2}
\begin{aligned}
\prob{\text{\textbf{SSFT+} fails}} &= \mathbb{P}((1 + \sum_{x \in M_i \setminus C} H_x) \hsrto{s}{2^{M_i}}\at{C} \\ &\quad\quad  +  \sum_{x \in N \setminus M_i} H_x \hsrto{s}{2^{M_i \cup \set{x}}}\at{C} = 0).
\end{aligned}
\end{equation}\end{linenomath*}
This probability is $=1$, if each of the partial Fourier coefficients 
\begin{linenomath*}\begin{equation}
\hsrto{s}{2^{M_i}}\at{C}, \hsrto{s}{2^{M_i \cup \set{x_{i+1}}}}\at{C},  \dots, \hsrto{s}{2^{M_i \cup \set{x_n}}}\at{C}
\end{equation}\end{linenomath*} 
is zero. Otherwise, the probability in \eqref{eq:prob_failure2} is zero, because the probability of a mixture of Gaussians taking a certain value is zero. Collecting those constraints for all relevant combinations of $M_i$ and $C \subseteq M_i$ yields the claim.
\end{proof}

Theorem~\ref{thm:alg2_pathological_appendix} shows that \textbf{SSFT+} correctly processes $\hsrto{s}{2^{M_i}}\at{C} = 0$ with $\P^{M_i}_C \neq \emptyset$, iff there is an element $x \in \set{x_{i+1}, \dots, x_n}$ for which $\hsrto{s}{2^{M_i \cup \set{x}}}\at{C} \neq 0$. Beyond that, it is easy to see that $\D_2 \subseteq \D_1$. However, in order to show that $\D_2$ is a proper subset of $\D_1$ we need Lemma~\ref{thm:K1_K2}.

%

\begin{lemma}\label{thm:K1_K2} For $M_i$ and $C \subseteq M_i$ with $\P^{M_i}_C \neq \emptyset$, we have either
\begin{enumerate}
\item $\K_1(M_i, C) \cap \S = \K_2(M_i, C) \cap \S = \emptyset$ or
\item $\K_1(M_i, C) \cap \S \neq \emptyset$ and $\K_2(M_i, C) \cap \S = \emptyset$ or
\item $\K_2(M_i, C) \cap \S \neq \emptyset$ and $\K_2(M_i, C) <_{\R} \K_1(M_i, C)$. 
\end{enumerate}
\end{lemma}
\begin{proof}
We prove with a case distinction: 

\textbf{Case 1:} $|\P^{M_i}_C| = 1$. In this case we have $\K_1(M_i, C) = \K_2(M_i, C)$ and, in particular, one of the $k$ coordinates, $\widehat{\mbf{s}}_1, \dots, \widehat{\mbf{s}}_k$, is required to be zero. Consequently, the intersection with $\S$ is empty: $\K_1(M_i, C) \cap \S = \K_2(M_i, C) \cap \S = \emptyset$.

\textbf{Case 2:} $|\P^{M_i}_C| \geq 2$ and there is an element $x \in \set{x_{i+1}, \dots, x_n}$ with $|\P^{M_i \cup \set{x}}_C| = 1$. In this case $\K_1(M_i, C)$ contains vectors that achieve $\sum_{B \in \P^{M_i}_C} \widehat{s}\at{B} = 0$ by cancellation. $\K_2(M_i, C)$ on the other hand, requires one coordinate to be zero.

\textbf{Case 3:} $|\P^{M_i}_C| \geq 2$ and all $x \in \set{x_{i+1}, \dots, x_n}$ with $\P^{M_i \cup \set{x}}_C \neq \emptyset$ and $\P^{M_i \cup \set{x}}_C \neq \P^{M_i}_C$ satisfy $2 \leq |\P^{M_i \cup \set{x}}_C| < |\P^{M_i}_C|$. That is, $\K_2(M_i, C)$ requires at least one additional equation to be satisfied in comparison to $\K_1(M_i, C)$. Therefore, $\K_2(M_i, C)$ is a subset of $\K_1(M_i, C)$. In addition, $\K_2(M_i, C)$ is closed under addition and scalar multiplication, which makes it a subspace of $\K_1(M_i, C)$.

There are no other cases, because if $|\P^{M_i}_C| \geq 2$ there is at least one element $x \in N \setminus M_i$ with $1 \leq |\P^{M_i \cup \set{x}}_C| < |\P^{M_i}_C|$. Such an element can be found constructively by taking the symmetric difference ${(B_1 \setminus B_2) \cup (B_2 \setminus B_1)}$ of $B_1 \neq B_2 \in \P^{M_i}_C$. The symmetric difference is non-empty because $B_1 \neq B_2$ and for all of its elements $x$ either ${x \in B_1 \land x \not\in B_2}$ or ${x \not\in B_1 \land x \in B_2}$ holds. Further, ${x \in (B_1 \setminus B_2) \cup (B_2 \setminus B_1) \subseteq N \setminus M_i}$ because ${B_1 \cap M_i = B_2 \cap M_i = C}$. Therefore, there is at least one $x \in N \setminus M_i$ with either ${B_1 \not\in \P^{M_i \cup \set{x}}_C \land B_2 \in \P^{M_i \cup \set{x}}_C}$ or ${B_1 \in \P^{M_i \cup \set{x}}_C \land B_2 \not\in \P^{M_i \cup \set{x}}_C}$.
\end{proof}

With Lemma~\ref{thm:K1_K2} in place we can now prove our main theorem of this section.

\begin{theorem} If $\D_1$ is non-empty, $\D_2$ is a proper subset of $\D_1$.
\end{theorem}
\begin{proof}
If $\D_1 \neq \emptyset$, there is at least one $M^{(1)} \in \set{M_1, \dots, M_n}$ and $C^{(1)} \subseteq M^{(1)}$ s.t. $|\P_{C^{(1)}}^{M^{(1)}}| \geq 2$. Otherwise, there would be no $M_i \subseteq N$ and $C \subseteq M_i$ with $\K_1(M_i, C) \cap \S \neq \emptyset$. $\K_1(M^{(1)}, C^{(1)})$ is $k - 1$ dimensional. Further, by Lemma~\ref{thm:K1_K2} we have $\dim K_2(M_i, C) \leq k - 2$, for all $M_i \subseteq N$ and $C \subseteq M_i$ with $\K_2(M_i, C) \cap \S \neq \emptyset$. Combining these two observations yields
\begin{linenomath*}\begin{equation}\label{eq:union_k2_subset_k1}
\begin{aligned}
\bigcup_{i = 0}^{n - 2} \bigcup_{\substack{C \subseteq M_i:\\ \K_2(M_i, C) \cap \S \neq \emptyset}} \K_2(M_i, C) \cap \K_1(M^{(1)}, C^{(1)}) \\\subset \K_1(M^{(1)}, C^{(1)}),
\end{aligned}
\end{equation}\end{linenomath*}
because the intersection of two subspaces is a subspace and a vector space cannot be written as the finite union of proper subspaces \cite{khare2009vector}. Now it remains to be shown that the intersection with $\S$ on both sides preserves the $\subset$ relation.

Observe that for two subsets $A, B \subseteq \R^k$ with $A \subset B$ we have $A \cap \S \subset B \cap \S \Leftrightarrow (B \setminus A) \cap \S \neq \emptyset$. For $A = $ LHS (left-hand side) and $B = $ RHS of \eqref{eq:union_k2_subset_k1}, we have $(B \setminus A) \cap \S \neq \emptyset \Leftrightarrow \S \setminus A \neq \emptyset$ because $B \cap \S \neq \emptyset$ by construction and $B \setminus A \neq \emptyset$ as $A \subset B$ according to \eqref{eq:union_k2_subset_k1}.

Now, we make use of $\S \setminus A = \S \cap A^c \neq \emptyset  \Leftrightarrow \S^c \cup A \neq \R^k$, which holds as $\R^k$ cannot be written as a finite union of proper subspaces. Recall that $\S^c$ is a finite union of subspaces (see \eqref{eq:S_complement}) and $A$ is the LHS of \eqref{eq:union_k2_subset_k1}.  
\end{proof}

\section{SSFT: Complexity}

To achieve the algorithmic and query complexity stated in the paper, we need to delve into the technical details of the implementation of \textbf{SSFT}. For mathematical convenience we introduce the notation $\B + x = \set{B \cup \set{x}: B \in \B}$ for sets of subsets $\B \subseteq 2^N$. Further, we observe that given $\B_{i-1} = \supp(\hsrto{s}{2^{M_{i-1}}})$, $\A_{i-1} = \set{M_{i-1} \setminus B: B \in \B_{i-1}}$ and $T^{i-1} = F^{-1}_{\A_{i-1}\B_{i-1}}$ we can construct the linear system that determines $\hsrto{s}{2^{M_{i}}}$ from $T^{i-1}$.

\begin{lemma}\label{lem:linear_system_propagation_model4} Let $\B = \B_{i-1} \cup (\B_{i-1} + x_i)$ and let $\A = \set{M_i \setminus B: B \in \B} = (\A_{i-1} + x_i) \cup \A_{i-1}$. We have ${\supp(\hsrto{s}{2^{M_i}}) \subseteq \B}$. Let 
\begin{linenomath*}\begin{equation}
\mbf{q}^{x_i} = (s\at{A})_{\A_{i-1} + x_i} \quad \text{and} \quad \mbf{q}^{\overline{x}_i} = (s\at{A})_{\A_{i-1}}.
\end{equation}\end{linenomath*}
Then  
  \begin{linenomath*}\begin{equation}
  F^{-1}_{\A \B} = \begin{pmatrix}
  T^{i-1} & 0 \\
  T^{i-1} & T^{i-1}
  \end{pmatrix}
  \end{equation}\end{linenomath*}
and the solution of 
  \begin{linenomath*}\begin{equation}
  \begin{pmatrix}
  T^{i-1} & 0 \\
  T^{i-1} & T^{i-1}
  \end{pmatrix} 
  \begin{pmatrix}
  \widehat{\mbf{q}}^{\overline{x}_i}\\
  \widehat{\mbf{q}}^{{x}_i}
  \end{pmatrix}
  =
  \begin{pmatrix}
  \mbf{q}^{x_i} \\
  \mbf{q}^{\overline{x}_i}
  \end{pmatrix}  
  \end{equation}\end{linenomath*}
  contains the Fourier coefficients of $\hsrto{s}{2^{M_i}}$ and is given by
  \begin{linenomath*}\begin{equation}
  \begin{aligned}
  \widehat{\mbf{q}}^{\overline{x}_i} &= (T^{i-1})^{-1} \mbf{q}^{x_i} \quad \text{and} \\ \widehat{\mbf{q}}^{{x}_i} &= (T^{i-1})^{-1} (\mbf{q}^{\overline{x}_i} - \mbf{q}^{x_i}).
  \end{aligned}
  \end{equation}\end{linenomath*}
\end{lemma} 
\begin{proof}
Let $\phi: 2^N \to \set{0, 1}^n \subseteq \R^n$ be the mapping from sets to indicator vectors, i.e., for $A \subseteq N$, $\phi(A)_i = 1$ if $x_i \in A$ and $\phi(A)_i = 0$ if $x_i \not\in A$. Let $\Phi$ denote the mapping from sets of subsets to indicator matrices, i.e., $\Phi(\A) = (\phi(A)^T)_{A \in \A} \in \set{0, 1}^{|\A| \times n} \subseteq \R^{|\A| \times n}$. Let, $$\rho: \R \to \R; a \mapsto \begin{cases}
1 & \text{if } a = 0,\\
0 & \text{else.}
\end{cases} $$ 
Let $\rho(R) = (\rho(R_{ij}))_{i,j=1}^{\ell}$, for a matrix $R$ in $\R^{\ell \times \ell}$. We observe 
\begin{linenomath*}\begin{equation}\label{eq:rho_rule}
\rho(R_1 + R_2) = \rho(R_1) \cdot \rho(R_2),
\end{equation}\end{linenomath*}
for matrices $R_1, R_2 \in \R^{\ell \times \ell}$ and $\cdot$ denoting the pointwise multiplication. 

With the introduced notation in place we obtain 
\begin{linenomath*}\begin{equation}
F^{-1}_{\A \B} = \rho(\Phi(\A)\Phi(\B)^T).
\end{equation}\end{linenomath*}

Now, we observe that 
\begin{linenomath*}\begin{equation}
\begin{aligned}
\Phi(\A) &= \begin{pmatrix}
\Phi(\A_{i - 1} + x_i)\\
\Phi(\A_{i - 1} )\hphantom{ + x_i)}
\end{pmatrix} \quad \text{and} \\
\Phi(\B) &= \begin{pmatrix}
\Phi(\B_{i - 1} )\hphantom{ + x_i)}\\
\Phi(\B_{i - 1} + x_i)
\end{pmatrix}.
\end{aligned}
\end{equation}\end{linenomath*}

As a consequence, we have
\begin{linenomath*}\begin{equation}
\tiny
\begin{aligned}
F^{-1}_{\A \B} &=
\rho \left(\begin{pmatrix}
\Phi(\A_{i - 1} + x_i) \Phi(\B_{i - 1} )^T & \Phi(\A_{i - 1} + x_i) \Phi(\B_{i - 1} + x_i)^T \\
\Phi(\A_{i - 1} )\hphantom{ + x_i)}\Phi(\B_{i - 1})^T & \Phi(\A_{i - 1} )\hphantom{ + x_i)}\Phi(\B_{i - 1} + x_i)^T
\end{pmatrix}\right)\\
&= 
\begin{pmatrix}
T^{i-1} & \rho(\Phi(\A_{i - 1}) \Phi(\B_{i - 1})^T)\cdot \rho(\mathbf{1}\mathbf{1}^T) \\
T^{i-1} & T^{i- 1} \hphantom{\rho(\Phi(\A_{i - 1}) \Phi(\B_{i - 1})^T)\cdot \rho()}
\end{pmatrix},
\end{aligned}
\end{equation}\end{linenomath*}
in which we applied \eqref{eq:rho_rule} and $\mbf{1}\mbf{1}^T$ denotes the $|\A_{i-1}| \times |\B_{i - 1}|$ matrix containing only ones. The claims now follows from Theorem~1 of \cited{puschel2020discrete} and because $\rho(\mbf{1}\mbf{1}^T)$ is the all-zero matrix. 
\end{proof}

\begin{remark}\label{lem:query_propagation_model4} We can reuse queries from iteration $i-1$ in iteration $i$, because
\begin{linenomath*}\begin{equation}
\mbf{q}^{\overline{x}_i} = (s\at{A})_{\A_{i-1}} = \mbf{q}_{i - 1}.
\end{equation}\end{linenomath*}
\end{remark}

The above results yield the following detailed implementation of \textbf{SSFT}:
\algnewcommand{\LineComment}[1]{\Statex \(\triangleright\) #1}

\begin{ssft}
\renewcommand{\thealgorithm}{}
\caption{Sparse set function Fourier transform of $s$ (detailed)}\label{alg:sdsft4_implementation}
\begin{algorithmic}[1]
    \LineComment{Perform initialization.}
	\State{$M_0 \gets \emptyset$}
	\State{$\mbf{q}_0 \gets (s\at{\emptyset})$}	
	\State{$\A_0 \gets \set{\emptyset}$}
	\State{$\B_0 \gets \set{\emptyset}$}
	\State{$T^0 \gets F^{-1}_{\mathcal{A}_{0}\mathcal{B}_{0}}$}
	\For{$i = 1, \dots, n$}
	
	\State{$M_i \gets M_{i-1} \cup \set{x_i}$}
	\State{$\mbf{q}^{\overline{x}_i} \gets \mbf{q}_{i-1}$}
	\LineComment{Notice that $\mbf{q}_{i-1} = (s\at{A})_{A \in \A_{i-1}}$.}
	\LineComment{Perform the new queries required.}	
	\State{$\mbf{q}^{x_i} \gets (s\at{A})_{A \in \A_{i-1} + x_i}$}
	\LineComment{Compute $\hsrto{s}{2^{M_i}}$ by solving two triangular systems.}
	\State{$\widehat{\mbf{q}}^{\overline{x}_i} \gets \text{ solve } T^{i-1} \widehat{\mbf{q}}^{\overline{x}_i} = \mbf{q}^{x_i} \text{ for } \widehat{\mbf{q}}^{\overline{x}_i}$}
	\State{$\widehat{\mbf{q}}^{x_i} \gets \text{ solve } T^{i-1} \widehat{\mbf{q}}^{x_i} = \mbf{q}^{\overline{x}_i} - \mbf{q}^{x_i} \text{ for } \widehat{\mbf{q}}^{{x}_i}$}
\LineComment{Construct the linear system for the next step.}
	\State{$\B_i^{\overline{x}_i} \gets \set{B \in \B_{i-1}: \mbf{q}_B^{\overline{x}_i} \neq 0}$}
	\State{$\B_i^{{x}_i} \gets \set{B \setminus \set{x_i}: B \in \B_{i-1} + x_i \land \mbf{q}_B^{{x}_i} \neq 0}$}
	\State{$\A_i^{\overline{x}_i} \gets \set{M_{i-1} \setminus B: B \in \B_i^{{x}_i}}$}	
	\State{$\A_i^{x_i} \gets \set{M_{i-1} \setminus B: B \in \B_i^{\overline{x}_i}}$}
	\State{$T^i \gets \begin{pmatrix}
	T^{i-1}_{\A_i^{x_i}\B_i^{\overline{x}_i}} & 0 \\
	T^{i-1}_{\A_i^{\overline{x}_i}\B_i^{\overline{x}_i}} & T^{i-1}_{\A_i^{\overline{x}_i}\B_i^{{x}_i}}	
	\end{pmatrix}$}
	\State{$\B_i \gets \B_i^{\overline{x}_i} \cup (\B_i^{x_i} + x_i)$}
	\State{$\A_i \gets (\A_i^{x_i} + x_i) \cup \A_i^{\overline{x}_i}$}
	\LineComment{Collect the queries required for the next step.}
	\State{$\mbf{q}_i \gets \begin{pmatrix}
	(\mbf{q}^{x_i}_A)_{A \in \A_i^{x_i} + x_i} \\
	(\mbf{q}^{\overline{x}_i}_A)_{A \in \A_i^{\overline{x}_i}\hphantom{ + x_i}}
	\end{pmatrix}$}
	\EndFor
	\LineComment{Read out the solution.}
	\For{$B \in \B_n$}
	\If{$x_n \in B$}
	\State{$\widehat{s}\at{B} \gets \widehat{\mbf{q}}^{{x}_n}_B$}
	\Else
	\State{$\widehat{s}\at{B} \gets \widehat{\mbf{q}}^{\overline{x}_n}_{B}$}
	\EndIf	
	\EndFor
	\State \Return{$\widehat{s}$}
    \end{algorithmic}
\end{ssft}

\begin{theorem}[\textbf{SSFT} number of queries]\label{thm:ssft_queries} \textbf{SSFT} requires at most $n k - k \log_2 k + 2 k = O(n k - k \log k)$ queries to reconstruct a $k$-Fourier-sparse set function in $\S \setminus \D_1$. 
\end{theorem}
\begin{proof}
In the worst case, the Fourier support is of a form for which \textbf{SSFT} has to perform the maximal amount of computation in each step. That is, up to (including) iteration $i = \lfloor \log_2 k \rfloor$ none of the Fourier coefficients $\hsrto{s}{2^{M_i}}$ are zero and from iteration $\lfloor \log_2 k \rfloor + 1$ exactly $k$ of the Fourier coefficients are non-zero. The ranges of the summations $\hsrto{s}{2^{M_i}}\at{B} = \sum_{A \subseteq N \setminus M_i} \widehat{s}\at{A \cup B}$ form a partition of $2^N$, thus, $\hsrto{s}{2^{M_i}}$ cannot be non-zero on more than $k$ sets. Therefore, for iteration $0$ (= initialization) one query is required, for iterations $1 \leq i \leq \lfloor \log_2 k \rfloor + 1$ exactly $2^{i-1}$ new queries are required (remember half of them are reused from iteration $i-1$) and for the remaining $n - (\lfloor \log_2 k \rfloor + 1)$ iterations $k$ new queries are required per iteration. Now, the claim follows from the simple derivation
\begin{linenomath*}\begin{equation}
\begin{aligned}
& 1 + \sum_{i = 1}^{\lfloor \log_2 k \rfloor + 1} 2^{i-1} + (n - \lfloor \log_2 k \rfloor - 1) k \\&= 1 + 2^{\lfloor \log_2 k \rfloor + 1} - 1 + (n - \lfloor \log_2 k \rfloor -1) k\\&\leq 2 k + (n - \lfloor \log_2 k \rfloor - 1) k\\&\leq n k - k \log_2 k + 2 k.
\end{aligned}
\end{equation}\end{linenomath*}
\end{proof}

\begin{theorem}[\textbf{SSFT} algorithmic complexity]\label{thm:ssft_ops} The algorithmic complexity of \textbf{SSFT} is $O(n k^2)$ for $k$-Fourier-sparse set functions in $\S \setminus \D_1$. 
\end{theorem}
\begin{proof}
The cost within a loop iteration of the loop in line~6 of \textbf{SSFT} is dominated by the cost of solving the two linear systems (in line~10 and line~11). Both linear systems are triangular and at most of size $k \times k$. Therefore, the cost of the body of the loop (line~6) is $O(k^2)$ (cost of solving a triangular linear system). We have to perform $n$ iterations of this loop resulting in the claim.
\end{proof}

\begin{theorem}[\textbf{SSFT+} query complexity] The query complexity of \textbf{SSFT+} is $O(n^2 k - n k \log k)$ for $k$-Fourier-sparse set functions in $\S \setminus \D_2$.
\end{theorem}
\begin{proof}
For a one-hop filter $h$, each evaluation of $h * s$ requires at most $n + 1$ queries from $s$. The claim now follows from Theorem~\ref{thm:ssft_queries}. 
\end{proof}

\begin{theorem}[\textbf{SSFT+} algorithmic complexity] The algorithmic complexity of \textbf{SSFT+} is $O(n^2 k + n k^2)$.
\end{theorem}
\begin{proof}
We require $O(n^2 k)$ operations in terms of queries and $O(n k^2)$ for solving the triangular linear systems.
\end{proof}

\section{Other Fourier Bases}

So far, we considered model 4 from Table~\ref{tab:sfdsp_appendix}. However, thanks to the algebraic viewpoint of \cited{puschel2020discrete}, both \textbf{SSFT} and \textbf{SSFT+} can be straightforwardly generalized to model~3 and model~5, by properly defining the chain of subproblems (with known superset of support) that need to be solved. Recall that for $x \in N\setminus 
M$, we computed ${\mathcal{B} \supseteq \supp(\hsrto{s}{2^{M \cup \set{x}}})}$ from $\supp(\hsrto{s}{2^{M}})$ to reduce the problem of computing the sparse Fourier transform (under mild conditions on the coefficients) to a chain of (solvable) subproblems. For model~4 this chain was
\begin{linenomath*}\begin{equation}\label{eq:dsft4_chain_appendix}
\srto{s}{2^{\emptyset}} = \hsrto{s}{2^{\emptyset}}, \hsrto{s}{2^{\set{x_1}}}, \hsrto{s}{2^{\set{x_1, x_2}}}, \dots, \hsrto{s}{2^{N}} = \widehat{s}.
\end{equation}\end{linenomath*}

In order to define the chains for model~3 and model~5, we first need to know how the Fourier transform of a restricted set function looks like under these models.

Let $M \subseteq N$ with $|M| = m$. Let $L \subseteq N \setminus M$. Let $s: 2^N \to \mathbb{R}$ be a set function. Let 
\begin{linenomath*}\begin{equation}\label{eq:restriction}
\srto{s}{L \cup 2^M}: 2^M \to \mathbb{R}; A \mapsto s\at{L \cup A}.
\end{equation}\end{linenomath*}

\begin{lemma}[Model~3]\label{thm:hashing_model3} Let $M^c = N \setminus M$. Using prior notation we have 
\begin{linenomath*}\begin{equation}\label{eq:hashing_model3}
\hsrto{s}{M^c \cup 2^M}\at{B} = \sum_{A \subseteq M^c} (-1)^{|A|} \widehat{s}\at{A \cup B}.
\end{equation}\end{linenomath*}
\end{lemma}
\begin{proof}
We have $\srto{s}{M^c \cup 2^M}\at{C} = s\at{M^c \cup C}$ per definition, for all $C \in 2^M$. Performing the Fourier expansion on both sides yields
\begin{linenomath*}\begin{equation}\label{eq:hashing_model3_ls}
\begin{aligned}
&\sum_{B \subseteq C} (-1)^{|B|} \hsrto{s}{M^c \cup 2^M}\at{B} \\&= \sum_{B \subseteq M^c \cup C} (-1)^{|B|} \widehat{s}\at{B} \\
&= \sum_{B \subseteq C} (-1)^{|B|} \sum_{A \subseteq M^c} (-1)^{|A|} \widehat{s}\at{A \cup B}.
\end{aligned}
\end{equation}\end{linenomath*}
\eqref{eq:hashing_model3} is the unique solution for the system of $2^m$ equations given by \eqref{eq:hashing_model3_ls}.
\end{proof}

\begin{lemma}[Model~5]\label{thm:hashing_model5} Using prior notation we have
\begin{linenomath*}\begin{equation}\label{eq:hashing_model5}
\hsrto{s}{2^M}\at{B} = \frac{1}{2^{n - m}}\sum_{A \subseteq N \setminus M} \widehat{s}\at{A \cup B}.
\end{equation}\end{linenomath*}
\end{lemma}
\begin{proof}
We have $\srto{s}{2^M}\at{C} = s\at{C}$ per definition, for all $C \in 2^M$. Performing the Fourier expansion on both sides yields
\begin{linenomath*}\begin{equation}\label{eq:hashing_model5_ls}
\begin{aligned}
&2^{-m} \sum_{B \subseteq M} (-1)^{B \cap C} \hsrto{s}{2^M}\at{B} \\&= 2^{-n}\sum_{B \subseteq N} (-1)^{B \cap C} \widehat{s}\at{B} \\
&= 2^{-n}\sum_{B \subseteq M} (-1)^{|B \cap C|} \sum_{A \subseteq N \setminus M} (-1)^{|A \cap C|} \widehat{s}\at{A \cup B} \\
&= 2^{-n}\sum_{B \subseteq M} (-1)^{|B \cap C|} \sum_{A \subseteq N \setminus M} \widehat{s}\at{A \cup B}.
\end{aligned}
\end{equation}\end{linenomath*}
\eqref{eq:hashing_model5} is the unique solution for the system of $2^m$ equations given by \eqref{eq:hashing_model5_ls}.
\end{proof}

Using Lemma~\ref{thm:hashing_model3} and Lemma~\ref{thm:hashing_model5} we can derive the corresponding chains of subproblems as follows.

\mypar{Model~3} For model~3, considering $\srto{s}{2^{M}}$ does not lead to a partitioning of all frequencies and, thus, does not lead to a support propagation rule. Instead, we consider $\srto{s}{M^c \cup 2^M}$. According to Lemma~\ref{thm:hashing_model3}, its Fourier transform has the desired form:
\begin{linenomath*}\begin{equation}
\hsrto{s}{M^c \cup 2^M}\at{B} = \sum_{A \subseteq M^c} (-1)^{|A|} \widehat{s}\at{A \cup B}.
\end{equation}\end{linenomath*}
Therefore, \textbf{SSFT} is obtained by considering the chain 
\begin{linenomath*}\begin{equation}
\srto{s}{N \cup 2^{\emptyset}} = \hsrto{s}{N \cup 2^{\emptyset}}, \hsrto{s}{N\setminus \set{x_1} \cup 2^{\set{x_1}}}, \dots, \hsrto{s}{\emptyset \cup 2^N} = \widehat{s}
\end{equation}\end{linenomath*}
in combination with Theorem~1 of \cited{wendler2019sampling}, which provides a solution for computing the Fourier coefficients w.r.t.~model~3 when the Fourier support is known. As models~1--4 all share the same frequency response \cite{puschel2020discrete}, we get \textbf{SSFT+} in the same way as for model~4.

\mypar{Model~5} For the WHT it is well known \cite{scheibler2013fast, amrollahi2019efficiently} (see Lemma~\ref{thm:hashing_model5}) that
\begin{linenomath*}\begin{equation}
\hsrto{s}{2^{M}}\at{B} = \frac{1}{2^{|N \setminus M|}}\sum_{A \subseteq N \setminus M} \widehat{s}\at{A \cup B}.
\end{equation}\end{linenomath*}
Therefore, we may use the same support propagation chain as in \eqref{eq:dsft4_chain_appendix}. Unfortunately, to the best of our knowledge, there is no theorem in the flavor of Theorem~1 of \cited{puschel2020discrete} or Theorem~1 of \cited{wendler2019sampling} for solving the corresponding systems of linear equations for the WHT. As a consequence, instead of solving a chain of linear systems we solve a chain of least squares problems (overdetermined linear systems) to obtain \textbf{SSFT} for the WHT. Note that it also may be possible to use more sophisticated methods in place of the least squares solution such as the one presented by \cited{stobbe2012learning}. \textbf{SSFT+} relaxes the requirements on the Fourier coefficients similarly as before. However, the analysis is slightly different because for the WHT the frequency response of a one-hop filter is  also computed with the WHT, i.e.,
\begin{linenomath*}\begin{equation}
\bar{h}\at{B} = 1 + \sum_{x \not\in B} h\at{\set{x}} - \sum_{x \in B} h\at{\set{x}}.
\end{equation}\end{linenomath*} 

Finally, we note that models~1 and~2 from \cited{puschel2020discrete} could be handled analogously.

\end{document}